\documentclass[11pt, a4paper]{article}
\usepackage[utf8]{inputenc}
\usepackage[left=3cm,right=3cm,top=3cm,bottom=3cm, a4paper]{geometry}
\usepackage{amsmath, amssymb, amsthm, enumerate, fixmath}
\usepackage[T2A]{fontenc}
\usepackage{tikz}
\usepackage{comment}
\usepackage{pgfplots}
\usepackage{graphicx}
\usepackage{xcolor}
\usepackage{tabularx}
\usepackage[english]{babel}
\usepackage{algorithm}
\usepackage{algpseudocode}
\usepackage{multirow}
\theoremstyle{plain}
\newtheorem{Th}{Theorem}[section]
\newtheorem{Lem}[Th]{Lemma}
\newtheorem{Prop}[Th]{Proposition}
\newtheorem{Cor}[Th]{Corollary}
\newtheorem{Prob}[Th]{Problem}
\usepackage{listings}
\theoremstyle{definition}
\newtheorem{Def}{Definition}[section]

\newtheorem{Ex}{Example}[section]

\newcommand{\cA}{{\mathfrak A}}
\newcommand{\cS}{{\mathfrak S}}

\newcommand{\cF}{{\mathfrak F}}

\newcommand{\norm}[1]{\left\lVert#1\right\rVert}

\newtheorem{rem}{Remark}[section]

\newtheorem{propt}{Property}[section]

\newcommand{\x}{{ x}}

\newcommand{\cL}{{\cal L}}

\newcommand{\bgeqn}{\begin{eqnarray}}
\newcommand{\edeqn}{\end{eqnarray}}
\newcommand{\bgeq}{\begin{eqnarray*}}
\newcommand{\edeq}{\end{eqnarray*}}
\newcommand{\bec}{\begin{center}}
\newcommand{\enc}{\end{center}}

\newcommand{\D}{{\cal D}}

\newcommand{\half}{ \mbox{\small$\frac{1}{2}$}}

\newcommand{\be}{\begin{equation}}
\newcommand{\ee}{\end{equation}}

\def\ess {{\rm ess\, sup}}
\def\esi {{\rm ess\, inf}}

\def\hq{\bar{q}}

\def\bbr{{\mathbb{R}}} %real numbers
\def\bbe{{\mathbb{E}}} %expectation
\def\bbn{{\mathbb{N}}}
\def\bbp{{\mathbb{P}}}

\def\bbq{{\mathbb{Q}}}

\makeatletter
\newsavebox{\@brx}
\newcommand{\llangle}[1][]{\savebox{\@brx}{\(\m@th{#1\langle}\)}%
  \mathopen{\copy\@brx\kern-0.5\wd\@brx\usebox{\@brx}}}
\newcommand{\rrangle}[1][]{\savebox{\@brx}{\(\m@th{#1\rangle}\)}%
  \mathclose{\copy\@brx\kern-0.5\wd\@brx\usebox{\@brx}}}
\makeatother

\newtheorem{remark}{Remark}

\usepackage [square,sectionbib]{natbib}

\def\plus{{\scriptscriptstyle +}} \def\minus{{\scriptscriptstyle -}}
\def\text#1{\;\,\hbox{#1}\;\,}    

\def\lset{\big\{\,}    \def\mset{\,\big|\,}   \def\rset{\,\big\}}
\def\Lset{\Big\{\,}    \def\Mset{\,\Big|\,}   \def\Rset{\,\Big\}}
\outer\def\proclaim #1. #2
   \par{\medbreak\noindent{\bf#1.\enspace}{\sl#2}\par
  \ifdim\lastskip<\medskipamount \removelastskip\penalty55\medskip\fi}

\def\paritem#1{\vskip0cm\noindent\hskip12pt{{\rm #1}}\hskip5pt}
\def\eop{\hfill{$\vcenter{\hrule height1pt \hbox{\vrule width1pt height5pt
   \kern5pt \vrule width1pt} \hrule height1pt}$} \medskip}
\def\low#1{{\lower1pt \hbox{$\scriptstyle #1$}}}
\def\high#1{{\raise1pt \hbox{$\scriptstyle #1$}}}
 
\def\implies{\quad\hbox{$\Longrightarrow$}\quad} 
\def\iff{\quad\hbox{$\Longleftrightarrow$}\quad}

\def\argmin{\mathop{\rm argmin}}   \def\argmax{\mathop{\rm argmax}}

\def\half{{{}\raise 1pt \hbox{$\frac{\scriptstyle 1}{\scriptstyle 2}$}}}
\def\eqalign#1{\begin{array}{lcr} #1 \end{array}}
\def\reals{{I\kern-.35em R}} \def\mdot{{\kern-.02em\cdot\kern-.04em}}

 \def\newpage{\vfill\eject}

\def\cA{{\cal A}}   \def\cD{{\cal D}} 
\def\cE{{\cal E}} \def\cF{{\cal F}}  \def\cI{{\cal I}} 
  \def\cL{{\cal L}}  
 \def\cQ{{\cal Q}} \def\cR{{\cal R}} \def\cS{{\cal S}} 
\def\cV{{\cal V}}

\def\<x>{\langle\!\langle\mathbf{x}\rangle\!\rangle}
\def\l<{\langle\!\langle}
\def\r>{\rangle\!\rangle}

\usepackage{hyperref} % remove borders around links in pdf
\hypersetup{
    colorlinks=true,
    linkcolor=blue,
    filecolor=magenta,      
    urlcolor=cyan,
    pdftitle={Overleaf Example},
    pdfpagemode=FullScreen,
    }

\usepackage{booktabs}
\pgfplotsset{compat=1.17}

\usepackage{subfigure}
\usepackage{wrapfig}
\usepackage{float}
\usepackage{multirow}
\usepackage{multicol}
\title{Support Vector Regression:\\
Risk Quadrangle Framework}
\author{Anton Malandii\thanks{Department of Applied Mathematics and Statistics, State University of New York, Stony Brook, NY 11794, USA. Email: \url{anton.malandii@stonybrook.edu}, \url{stanislav.uryasev@stonybrook.edu}}
\and
Stan Uryasev\footnotemark[1]
}

\begin{document}
\maketitle

\begin{abstract}

%This paper investigates Support Vector Regression (SVR) in the framework of the Fundamental Risk Quadrangle (RQ) theory. Every RQ includes four stochastic functionals (error, regret, risk, and deviation) bound together by a so-called statistic. The RQ links stochastic optimization, risk management, and statistical estimation into a unified theory. In this framework,
%$\varepsilon$-SVR and $\nu$-SVR are reduced to minimization of \emph{Vapnik error} and the conditional value-at-risk (CVaR) norm, accordingly.  
 %The Vapnik error and CVaR norm define quadrangles with a statistic equal to an average of two symmetric quantiles. Therefore, RQ theory implies that both $\varepsilon$-SVR and $\nu$-SVR are
%asymptotically unbiased estimators of an average of two symmetric conditional quantiles. The equivalence between $\varepsilon$-SVR and $\nu$-SVR is demonstrated in a general stochastic setting. Additionally, SVR is formulated as a deviation minimization problem with a regularization penalty. Furthermore, the integration of $\nu$-SVR into the RQ allowed the formulation of it as a distributionally robust regression problem. Finally, an alternative dual formulation of
%SVR in the RQ framework is derived. Theoretical results are validated with a case study.
\noindent
This paper investigates \emph{Support Vector Regression} (SVR) within the framework of the \emph{Risk Quadrangle} (RQ) theory. Every RQ includes four stochastic functionals -- \emph{error, regret, risk}, and \emph{deviation}, bound together by a so-called statistic. The RQ framework unifies stochastic optimization, risk management, and statistical estimation. Within this framework, both $\varepsilon$-SVR and $\nu$-SVR are shown to reduce to the minimization of the \emph{Vapnik error} and the \emph{Conditional Value-at-Risk} (CVaR) norm, respectively. The Vapnik error and CVaR norm define quadrangles with a statistic equal to the average of two symmetric quantiles. Therefore, RQ theory implies that $\varepsilon$-SVR and $\nu$-SVR are asymptotically unbiased estimators of the average of two symmetric conditional quantiles. Moreover, the equivalence between $\varepsilon$-SVR and $\nu$-SVR is demonstrated in a general stochastic setting.
Additionally, SVR is formulated as a deviation minimization problem. Another implication of the RQ theory is the formulation of $\nu$-SVR as a Distributionally Robust Regression (DRR) problem. Finally, an alternative dual formulation of SVR within the RQ framework is derived. Theoretical results are validated with a case study.
\vspace{0.2cm}


\end{abstract}

\section{Introduction}
\emph{Regression} approximates a
random variable $Y$ by a function $\hat{f}$ of an observed random vector $ \mathbold{X} = (X_1,\ldots, X_n)^\top$. The function $\hat{f}(\mathbold{X})$ from a given class  $\cF$ is found by minimizing an error function applied to a regression residual $Z_f = Y-f(\mathbold{X})$. Usually, a norm serves as an error (e.g., $\cL^1$-regression, $\cL^2$-regression), however, more generally, (cf. \citep{RiskTuning}), axiomatically defined error measures can be used.
Usually, class  $\cF$ consists of polynomials, splines, wavelets, or neural networks (see e.g., \citep{MLBook}). 

From the statistical perspective, the purpose of regression is to \emph{estimate} a conditional \emph{statistic} $\cS(Y|\mathbold{X})$ of a random variable $Y$ given $\mathbold{X}$ by finding a function $\hat{f}$, which is called the \emph{best estimator} (or Bayes predictor according to \cite{Bach}). For example, $\cL^1$-regression estimates the conditional median, i.e.,  $\hat{f}_{\cL^1}(\mathbold{X}) \in \textrm{med}[Y|\mathbold{X}],$ (where we use ``$\in$'' to emphasize that the optimal solution may not be unique) and $\cL^2$-regression estimates the conditional mean $\hat{f}_{\cL^2}(\mathbold{X}) = \bbe[Y|\mathbold{X}].$

There is an extensive literature related to regression. For instance, Google search ``linear regression'' results in 323 million hits (on September 27, 2023). Here we refer only two general frameworks directly relevant to Support Vector Regression (SVR): 1) VC theory, popular in the machine learning community, \citep{VapnikBook} and 2) RQ theory, well-known in risk management, \citep{Quadrangle}. The first framework considers error measures only of the expectation type, i.e., $\bbe [\ell(Z_f)]$, where $\ell$ is a so-called loss function; hence the choice of error boils down to the choice of the loss function. The second framework considers axiomatically defined error measures $\cE(Z_f)$ that are not necessarily of expectation type. Moreover, RQ links a selected error with other uncertainty measures: risk, deviation, and regret. A brief introduction to the RQ theory and its relationship formulae for quadrangle construction can be found in  Appendix \ref{appendix1}.
Errors of expectation type play a crucial role in estimation theory. Indeed, the \cite[Regression Theorem]{Quadrangle} states that in this case, the best estimator belongs to a conditional statistic from the corresponding quadrangle, i.e.,
$\hat{f}_\cE(\mathbold{X}) \in \cS(Y|\mathbold{X}).
$
Hence the choice of a loss function results in a particular statistic. 

In the context of machine learning, regression is understood as a procedure for an optimal fitting of a given dataset $\left(\mathbold{x}_i, y_i\right)_{i=1}^l, \ \mathbold{x}_i \in \bbr^n, \ y_i \in \bbr$, which is called a \emph{training sample}. The goal here is constructing a prediction model $\hat{f}(\mathbold{x})$, which gives a forecast $\hat{f}(\mathbold{x}_{l+1})$ for the future outcome $y_{l+1}$. Assuming that an
error function $\cE(Z_f)$ is minimized,  \citep[Regression Thorem]{Quadrangle} implies that
$$\hat{f}_\cE(\mathbold{x}_{l+1}) \in \cS(Y|\mathbold{X} = \mathbold{x}_{l+1}),
$$
i.e., the optimal prediction model is a conditional statistic.

This paper studies SVR using the RQ framework. SVR is a well-established machine learning method that has been extensively studied in the framework of VC theory. However, important questions are not addressed with VC theory. In particular, it is not clear which statistical quantity SVR estimates.

\paragraph{Contribution.}
By formulating SVR within the RQ framework, we establish a connection between this machine learning approach and classical statistics, risk management, and distributionally robust optimization (DRO). We derive a quadrangle corresponding to $\varepsilon$-SVR (see Proposition~\ref{Prop expect quadr}), where the risk and deviation can be used for risk management.

We demonstrate that SVR is an asymptotically unbiased estimator of the average of two symmetric conditional quantiles (see Figure \ref{fig:SVR graph} for a graphical illustration)
$$\cS_\alpha(Y|\mathbold{X}) = \dfrac{1}{2}\big(q_{(1-\alpha)/2}\big(Y|\mathbold{X}\big)+ q_{(1+\alpha)/2}\big(Y|\mathbold{X}\big)\big) \, ,
$$
see Subsection \ref{sec: estimation}. 
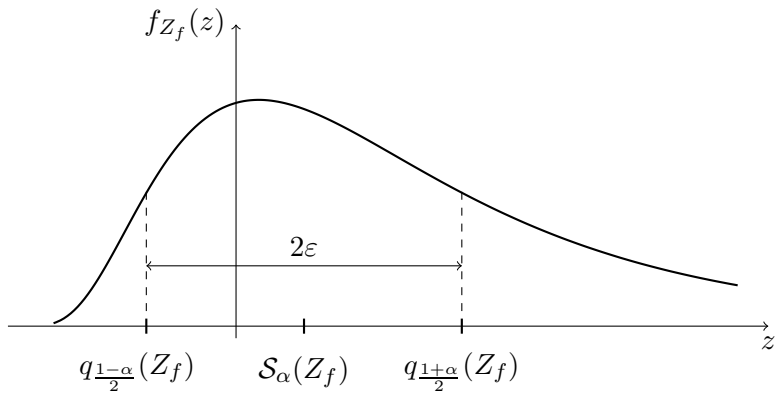
\begin{figure}[ht!]
    \centering
   \begin{tikzpicture}[scale=2]
    % x-axis and labels
    \draw[->] (-1,0) -- (4,0) node[anchor=north] {$z$};
    \draw[->] (0.5,-0.08) -- (0.5,2) node[anchor=east] {$f_{Z_f}(z)$};

    % Shifted lognormal density curve
    \draw[thick, smooth, domain=-0.7:3.8, samples=100] 
        plot(\x, {1.5 * exp(-1.5 * ((ln(\x+1)-0.5)^2))});

    % Mean line
    \draw[thick] (0.947,-0.05) -- (0.947,0.05); \node at (0.947,-0.3) {$\cS_\alpha(Z_f)$};

    % Quantiles
    \draw[dashed] (-0.09,0.1) -- (-0.09,0.9);  \draw[thick] (-0.09, -0.05) -- (-0.09, 0.05);
    \node at (-0.15, -0.3) {$q_{\frac{1-\alpha}{2}}(Z_f)$};

    \draw[dashed] (1.9847,0.1) -- (1.9847,0.9);\draw[thick] (1.9847, -0.05) -- (1.9847, 0.05);
    \node at (1.9847, -0.3) {$q_{\frac{1+\alpha}{2}}(Z_f)$};

    % Double arrow for distance x
    \draw[<->] (-0.09, 0.4) -- (1.9847, 0.4) node[midway, anchor=south] {$2\varepsilon$};
\end{tikzpicture}
    \caption{Graphical illustration of the SVR estimator. For $\varepsilon$-SVR, let $\mathbb{E}[|Z_f| - \varepsilon]_+$ be the Vapnik error with parameter $\varepsilon \geq 0$, where $Z_f$ denotes the regression residual with pdf $f_{Z_f}(z)$. The parameter $\varepsilon$ defines the distance between two symmetric quantiles, $q_{\frac{1+\alpha}{2}}(Z_f)$ and $q_{\frac{1-\alpha}{2}}(Z_f)$, such that $\varepsilon = \frac{1}{2}\Big(q_{\frac{1+\alpha}{2}}(Z_f) - q_{\frac{1-\alpha}{2}}(Z_f)\Big)$. The optimal solution to the regression problem with the Vapnik error is then the average of these two conditional quantiles, where the parameter $\alpha$ is implicitly defined by $\varepsilon$. For $\nu$-SVR with $\nu = 1 - \alpha$, one specifies the parameter $\alpha$ directly and minimizes the CVaR norm.}
    \label{fig:SVR graph}
\end{figure}

This implies that by adjusting the parameter $\alpha = 1-\nu$ (or $\varepsilon$), $\nu$-SVR (or $\varepsilon$-SVR) can estimate various distributional statistics, such as the mean, median, and expectiles. The desired statistic can be estimated by selecting an appropriate performance metric (error function) in cross-validation and tuning the parameter $\nu$ (or $\varepsilon$) accordingly. For example, choosing Mean Squared Error (MSE) in cross-validation results in a good estimate of the mean, while Asymmetric Mean Squared Error (AMSE) is appropriate for expectiles, \citep{Newey} (see Lemma \ref{lemma}). %More importantly, understanding SVR estimation properties reveals its limitations (see Example \ref{counterexample}).

Another result of RQ theory is the formulation of SVR as a deviation minimization problem (see Corollary \ref{error shap coroll}). This has both conceptual and practical implications. In the case of linear regression, it reduces the dimensionality of the problem, as the intercept can be calculated analytically.  Also, this approach addresses issues with calculating an optimal regression intercept, \citep{SVRTutorial}. Conceptually, RQ theory reveals the deviation and risk measures associated with SVR. 

We prove the equivalence of $\nu$-SVR and $\varepsilon$-SVR in a general stochastic setting, providing analytical expressions for $\varepsilon$ and $\nu$ that establish this equivalence (see Proposition \ref{dual svr connection}).

\cite{stable_regres} introduced the concept of stable regression as a robust approach to training regression models. By replacing conventional random data assignment with an optimization-based method, they achieved significantly improved prediction accuracy, greater model stability, and more effective feature selection. In turn, by leveraging the duality theory of convex functionals within the RQ framework, we reformulate $\nu$-SVR as a distributionally robust regression (DRR) problem (see Subsection~\ref{sec:svr as drr}) and prove its equivalence to stable regression.

Finally, we derive a new dual formulation of SVR within the RQ framework (see Subsection \ref{sec:dual formul}). This new formulation offers computational advantages, halving the number of variables compared to the standard dual formulation. It is mathematically transparent and can be solved using general-purpose optimization packages such as CPLEX, Gurobi, CVX, PSG, etc., allowing for efficient implementation and flexibility in choosing optimization tools. 

Overall, this paper provides a comprehensive analysis of SVR, extending and generalizing key results, including the equivalence of SVR formulations, its interpretation as DRR, and its general dual formulation. Additionally, it rigorously demonstrates what SVR estimates, highlighting both its capabilities and limitations.

\paragraph{Outline.} Section \ref{Sec SVR Formulations} reviews  formulations of $\varepsilon$-SVR and $\nu$-SVR, presenting their equivalent stochastic reformulations. Section \ref{sec: Risk Quadrangle Framework} introduces key definitions and theorems related to the RQ framework and its relation to SVR. In particular, it constructs a new quadrangle corresponding to the Vapnik error. Section \ref{sec: svr as generalized regression} studies SVR as a generalized regression problem in the RQ framework. Specifically, it proves the equivalence between $\varepsilon$-SVR and $\nu$-SVR, formulates SVR as a deviation minimization problem, discusses its estimation properties, interprets SVR as a form of DRR, and presents a novel dual formulation of SVR. Section \ref{sec: case studies} presents a case study based on simulated data, providing numerical verification of the paper's theoretical findings. Finally, Section~\ref{sec:conclusion} summarizes the key contributions and results of the paper.

\section{SVR Formulations}
\label{Sec SVR Formulations}
This section formulates two popular SVRs: $\varepsilon$-SVR and $\nu$-SVR. Equivalence of $\varepsilon$-SVR and $\nu$-SVR was established in \cite{NuSVR} in the sense that for $\nu$-SVR  with parameter $\nu \in (0,1]$ there exists an $\varepsilon \geq 0$ such that $\varepsilon$-SVR has the same optimal solution.
\subsection{The $\texorpdfstring{\mathbold{\varepsilon}}{e}$-SVR}
Consider a linear regression with a training sample 
\begin{equation}\label{empirical data}
    X^l = \left(\mathbold{x}_i,y_i\right)_{i=1}^l,
\end{equation}
where $\mathbold{x}_i \in \bbr^n$ is a \emph{feature} vector and  $y_i \in \bbr$ is a \emph{target output}. 
One needs to find a hyperplane $ \mathbold{w}^\top\mathbold{x} + b, \; (\mathbold{w},b) \in \bbr^{n+1}$ 
that optimally fits the given training data. This problem can be efficiently solved with $\varepsilon$-SVR introduced by \cite{VapnikBook}. The $\varepsilon$-SVR is formulated as follows
\begin{equation}\label{e-SVR with e-loss vector}
   \min_{\mathbold{w},b} \quad \frac{1}{l}\sum_{i=1}^{l} [|y_i - \mathbold{w}^\top\mathbold{x}_i - b | - \varepsilon]_+ + \frac{\lambda}{2}\norm{\mathbold{w}}^2_2 \, , 
\end{equation}
where $\lambda>0$, $\varepsilon > 0$, $\norm{\cdot}_2$ denotes the $\ell^2$-norm\ , and
$[a]_+ = \max\{0,a\}$  is a positive part of a number $a \in \bbr$.
%We rewrite (\ref{e-SVR with e-loss}) with slightly different notations, frequently adopted in machine learning literature:
%\begin{equation}\label{e-SVR with e-loss vector}
 %   \min_{\mathbold{w},b} \quad   \frac{1}{l}\textbf{1}^\top_l\mathbold{L}_{\varepsilon} (\mathbold{y} - \hat{\textbf{X}}\mathbold{w} - \textbf{1}_lb ) + \frac{\lambda}{2}\norm{\mathbold{w}}^2_2,
%\end{equation}
 The $\varepsilon$-SVR searches for a hyperplane having at most $\varepsilon$ deviation from targets $y_i$. The constant $\lambda>0$ determines a trade-off between the flatness (magnitude $\|\mathbold{w}\|_2^2$ of the weight vector $\mathbold{w}$) of the hyperplane and the amount up to which deviations larger than $\varepsilon$ are tolerated.

In the probabilistic framework, we consider that $\mathbold{z} = \mathbold{z}(\mathbold{w},b) $
is a random variable taking with equal probabilities components of the vector
$ (y_1 -\mathbold{w}^\top\mathbold{x}_1 - b, \ldots, y_l - \mathbold{w}^\top\mathbold{x}_l - b)^\top$
and define the expected loss as
$$
\bbe[|\mathbold{z}(\mathbold{w},b)|-\varepsilon]_+=\frac{1}{l}\sum_{i=1}^{l} [|y_i - \mathbold{w}^\top\mathbold{x}_i - b | - \varepsilon]_+ \, .
$$
Further, we reformulate $\varepsilon$-SVR (\ref{e-SVR with e-loss vector}) as follows
\begin{equation}\label{e-SVR with e-loss vector probabilistic}
    \min_{\mathbold{w},b} \quad   \bbe[|\mathbold{z}(\mathbold{w},b)|-\varepsilon]_+ + \frac{\lambda}{2}\norm{\mathbold{w}}^2_2\,.
\end{equation}

\subsection{The $\texorpdfstring{\mathbold{\nu}}{nu}$-SVR}
$\nu$-SVR introduced by \cite{NewSVM}, can be formulated as follows
\begin{equation}\label{nu-SVR with e-loss vector probabilistic}
    \min_{\mathbold{w},b, \varepsilon} \quad   \bbe[|\mathbold{z}(\mathbold{w},b)|-\varepsilon]_+ +\varepsilon \nu + \frac{\lambda}{2}\norm{\mathbold{w}}^2_2 \, ,
\end{equation}
 where parameter $\nu \in (0,1]$ controls the number of support vectors.  Similar to \cite{NuSVMasCVaR} we reformulate the $\nu$-SVR.  

The minimum w.r.t $\epsilon$ of the first two terms in the previous formula  equals
$$
    \min_{\varepsilon} \quad   \bbe[|\mathbold{z}(\mathbold{w},b)|-\varepsilon]_+ +\varepsilon \nu\; = \;\nu \hq_{1-\nu}(|\mathbold{z}(\mathbold{w},b)|) 
    	\;\equiv \;\llangle \mathbold{z}(\mathbold{w},b)) \rrangle_{1-\nu}\,\,,
$$
where $\hq_{1-\nu}(\cdot)$ is the conditional value-at-risk (CVaR), 
see Definition \ref{squantile}, and $\llangle \cdot \rrangle_{1-\nu}$ is  CVaR norm
studied  in \cite{CVaRNorm,bertsimas2011robust}, see   Definition~\ref{non-scaled CVaR Norm}.
 Therefore, $\nu$-SVR \eqref{nu-SVR with e-loss vector probabilistic} is reformulated as follows 
\begin{equation}\label{nu-SVR with cvar norm vector probabilistic}
    \min_{\mathbold{w},b} \quad   \llangle \mathbold{z}(\mathbold{w},b) \rrangle_{1-\nu} + \frac{\lambda}{2}\norm{\mathbold{w}}^2_2\,.
\end{equation}

\section{Risk Quadrangle Framework}\label{sec: Risk Quadrangle Framework}
This section formally introduces quantile (also called value-at-risk (VaR) in finance), CVaR, CVaR norm, Vapnik error, and related quadrangles.

\subsection{CVaR and Optimization Formulas}
We consider risk measures (stochastic functionals ranking random values) satisfying the following properties: \emph{constant neutrality, convexity, aversity, closedness, monotonicity, and homogeneity}. A risk measure that possesses the first four properties is referred to as a \emph{regular} risk measure.

Let $(\Omega, \mathcal{A}, \mathbb{P})$ be a probability space, $X \in \cL^2(\Omega)$ be a real-valued random variable, and the cumulative distribution function be denoted by $F_X(x) = \bbp(X\leq x), \ x \in \bbr$.
\begin{Def}[Regular Risk Measure, \citep{Quadrangle}]\label{regular risk measure}
A functional $\cR: \cL^2(\Omega) \to \bbr \cup \{+ \infty\} $ is called a \textit{regular measure of risk} if it satisfies the following axioms
\begin{itemize}
    \item[(R1)] \textbf{constant neutrality:} $\cR(C) = C, \quad \forall \; C = const\,;$
     \item[(R2)] \textbf{convexity:} $\cR\left(\lambda X + (1-\lambda)Y\right) \leq \lambda \cR(X) + (1-\lambda)\cR(Y), \quad \forall \; X,Y$ and $\lambda \in [0,1]\,$;
    \item[(R3)] \textbf{closedness:} $\left\{ X \in \cL^2(\Omega) \mset \cR(X) \leq c\right\}$ is closed $\forall \; c < \infty\,$;
    \item[(R4)] \textbf{aversity:} $\cR(X) > \bbe X, \quad \forall \; X \neq const\,.$
\end{itemize}
\end{Def}

\begin{Def}[Quantile]\label{quantile}
The \textit{quantile} (value-at-risk or VaR) of a random variable X at confidence level $\alpha \in [0,1]$ is a set defined as follows
\begin{equation}
    q_\alpha(X) = [q_\alpha^- (X), q_\alpha^+ (X)] \, ,
\end{equation}
where
\begin{equation}\label{quantile-}
    q_\alpha^- (X) = \begin{cases}
 \sup \left\{ x\mset F_X(x) < \alpha\right\}, &\alpha \in (0,1] \\
 \esi(X), &\alpha = 0
 
\end{cases}
\end{equation}
\begin{equation}\label{quantile+}
    q_\alpha^+ (X) = \begin{cases}
 \inf \left\{ x\mset F_X(x) > \alpha\right\},&\alpha \in [0,1) \\
 \ess(X),&\alpha = 1
 
\end{cases}
\end{equation}
If $q_\alpha^-(X) = q_\alpha^+(X)$ then 
\begin{equation*}
    q_\alpha(X) = q_\alpha^-(X) = q_\alpha^+(X) \, .
\end{equation*}
\end{Def}
\begin{rem}[Sum and scaling of quantiles]
Quantile, generally speaking, is an interval, therefore, the sum of two quantiles is defined as a Minkowski sum of convex sets, i.e., for $\alpha_1,\alpha_2 \in [0,1]$
\begin{equation}\label{sum of quants}
    q_{\alpha_1}(X) + q_{\alpha_2}(X) = \lset v+w \mset v \in  q_{\alpha_1}(X), w \in q_{\alpha_2}(X) \rset\,.
\end{equation}
The scaling of a quantile by an arbitrary constant $\lambda \in \bbr$ is defined as follows
\begin{equation}\label{scaling of quants}
    \lambda q_{\alpha}(X) = \lset \lambda w \mset w \in q_{\alpha}(X) \rset\,.
\end{equation}
\end{rem}

CVaR (also called, tail value-at-risk, average value-at-risk, expected shortfall) is a popular regular risk measure. It has favorable mathematical properties and can be efficiently optimized, \citep{CVaR, CVaR2}.
\begin{Def}[CVaR]\label{squantile}
The \emph{Conditional Value-at-Risk (CVaR)} of random variable X at confidence level $\alpha \in [0,1]$ is defined as 
\begin{equation}
    \hq_{\alpha}(X) = \dfrac{1}{1-\alpha}\displaystyle \int\limits_{\alpha}^{1}q_{\beta}(X) \, d\beta, \quad \alpha \in (0,1)\,.
\end{equation}
For $\alpha = 0:$ 
$
   \;\; \hq_0(X) = \lim_{\varepsilon \to 0} \hq_{\varepsilon}(X) = \bbe[X]\,. 
$

\noindent
For $\alpha = 1:$
$
  \;\;  \hq_1(X) = \ess(X)\,. 
$
\end{Def}
The following Theorem \ref{squantile opt th} by \cite{CVaR2} is used to build optimization algorithms for CVaR.
\begin{Th}[CVaR Optimization Formula] \label{squantile opt th}
For a random variable $X$ and $\alpha \in (0,1)$, 
\begin{equation}\label{squantile opt formula}
    \hq_\alpha(X) = \min_C \left\{ C + \frac{1}{1-\alpha}\bbe [X-C]_+\right\}\,,
\end{equation}
and the set of minimizers for (\ref{squantile opt formula}) is $q_\alpha(X)$.
\end{Th}
There is a deep relation between CVaR and the mean excess function $\bbe[X-x]_+, \  x \in \bbr$ (the mean excess function is also called regret or partial moment of order $1$). 
%For the \emph{superexpectation}
%\begin{equation}
%  \bbe[X-x]_+ +x = \bbe_X(x)\,,  
%\end{equation}
\cite{rockafellar2014random} proved the following theorem (see, also, \citep{DualCVaR}).
\begin{Th}[Dual CVaR Optimization Formula]\label{dual squantile th}
For a random variable $X$ and $x \in \bbr$, 
\begin{equation}\label{dual squantile formula}
    \mathbb{E}[X-x]_+ = \max_{\alpha \in [0,1]}\;(1-\alpha)(\hq_{\alpha}(X)-x)\,,
\end{equation}
and the set of maximizers for (\ref{dual squantile formula}) is $[\bbp(X<x), \bbp(X\leq x)]\,.$
\end{Th}
\begin{comment}
\begin{proof}
See, \nameref{Th3.2}.
\end{proof}    
\end{comment}
\begin{rem}\label{dual superquantile remark}
Note that $(1-\alpha)\hq_\alpha(X)$ is a concave function of $\alpha$, \citep{CVaR2}. Therefore, (\ref{dual squantile formula}) is a concave maximization problem.
\end{rem}

\subsection{CVaR Norm and Related Quadrangles}
Axiomatic analysis of general measures of error was introduced and developed by \cite{RiskTuning}. CVaR norm, considered by \cite{CVaRNorm, bertsimas2011robust} in $\bbr^n$ and extended by \cite{PICHLER2013405, CVaRNorm2} to infinite-dimensional setting, is a particular case of a regular measure of error. 
\begin{Def}[Regular Error Measure, \citep{Quadrangle}]\label{reg error measure}
A functional $\cE: \cL^2(\Omega) \to \bbr^+ \cup \{+ \infty\} $ is called a \textit{regular measure of error} if it satisfies the following axioms:
\begin{itemize}
    \item[(E1)] \textbf{zero neutrality:} $\cE(0) = 0$;
    \item[(E2)] \textbf{convexity:} $\cE\left(\lambda X + (1-\lambda)Y\right) \leq \lambda \cE(X) + (1-\lambda)\cE(Y), \quad \forall \; X,Y$ and $\lambda \in [0,1]$;
    \item[(E3)] \textbf{closedness:} $\left\{ X \in \cL^2(\Omega) \mset \cE(X) \leq c\right\}$ is closed $\forall \; c < \infty$;
    \item[(E4)] \textbf{nonzeroness:} $\cE(X) > 0, \quad \forall \; X \neq 0$.
\end{itemize}
\end{Def}
\begin{Def}[Scaled CVaR Norm]\label{Scaled CVaR Norm}
Let $X \in \cL^1(\Omega)$ be a real-valued random variable. Then \emph{scaled CVaR norm} of $X$ with parameter $\alpha \in [0,1]$ is defined by 
\begin{equation}\label{scaled squant norm}
    \llangle X \rrangle_\alpha^S = \hq_\alpha\left(|X|\right)\,.
\end{equation}
\end{Def}
When referring to the CVaR norm, we assume its \emph{scaled} version. However, following the \cite{CVaRNorm}  below we define an equivalent \emph{non-scaled} CVaR norm.
\begin{Def}[Non-scaled CVaR Norm]\label{non-scaled CVaR Norm}
Let $X \in \cL^1(\Omega)$ be a real-valued random variable. Then \emph{non-scaled CVaR norm} of $X$ with parameter $\alpha \in [0,1)$ is defined by 
\begin{equation}\label{non-scaled squant norm}
    \llangle X \rrangle_\alpha = (1-\alpha)\hq_\alpha\left(|X|\right).
\end{equation}
\end{Def}
\cite{CVaRNorm2} proved the following Proposition \ref{Prop Mafusalov}, defining the CVaR Norm Quadrangle.
\begin{Prop}[CVaR Norm Quadrangle]\label{Prop Mafusalov}
For $\alpha \in [0,1)$ the error measure $\cE_\alpha(X) =  \llangle X \rrangle_\alpha$ generates the following regular quadrangle (see Definition \ref{risk quadrangle}):
\begin{equation*}
    \begin{aligned}
        \cS_\alpha(X) &= \frac{1}{2}\left(q_{(1-\alpha)/2}(X) + q_{(1+\alpha)/2}(X)\right)\,,\\
        \cR_\alpha(X) &= \frac{1}{2}\big((1+\alpha)\hq_{(1-\alpha)/2}(X)+ (1-\alpha)\hq_{(1+\alpha)/2}(X)\big)\,,\\
        \cD_\alpha(X) &= \frac{1}{2}\big((1+\alpha)\hq_{(1-\alpha)/2}(X-\bbe[X])+ (1-\alpha)\hq_{(1+\alpha)/2}(X-\bbe [X])\big)\,,\\
        \cV_\alpha(X) &= \llangle X \rrangle_\alpha + \bbe [X]\,,\\
        \cE_\alpha(X) &= \llangle X \rrangle_\alpha\,.
    \end{aligned}
\end{equation*}
\end{Prop}
 We call by Vapnik error, the error defined by $\cE_\varepsilon(X) =  \bbe[|X|-\varepsilon]_+$. This name is inspired by Vapnik's $\varepsilon$-insensitive loss function $\ell(\xi) = [|\xi|-\varepsilon]_+$. The following Proposition~\ref{Prop expect  quadr}  presents a quadrangle based on the Vapnik error. It is closely related to the CVaR Norm Quadrangle considered in Proposition \ref{Prop Mafusalov}.

\begin{Prop}[Quantile Symmetric Average  Quadrangle]\label{Prop expect quadr}Let $X \in \cL^2(\Omega),$\\ $ 0 \leq \varepsilon < \frac{1}{2}(\ess\, X - \esi\,X), $ and  $(\cR_\alpha, \cD_\alpha, \cV_\alpha, \cE_\alpha)$ be the CVaR Norm Quadrangle quartet with statistic $\cS_\alpha$. Then the set
\begin{equation}\label{extrema condition}
    \begin{aligned}
     \cA_\varepsilon(X) = \Lset \alpha \Mset  & \varepsilon \in \frac{1}{2}\bigl (q_{(1+\alpha)/2}(X) - q_{(1-\alpha)/2}(X) \bigr )\Rset
    \end{aligned}
\end{equation}
%\begin{equation}\label{extrema condition}
%    \begin{aligned}
%     \cA_x(X) = \Lset \alpha \Mset  &\frac{1}{2}(q^-_{(1+\alpha)/2}(X) - q^+_{(1-\alpha)/2}(X)) \leq x \leq \frac{1}{2}(q^+_{(1+\alpha)/2}(X) - q^-_{(1-\alpha)/2}(X))\Rset
 %   \end{aligned}
%\end{equation}
is nonempty and the Vapnik error  generates the following quadrangle:
\begin{equation*}
    \begin{aligned}
        \cS_\varepsilon(X) &= \bigcup\limits_{\alpha \in \cA_\varepsilon(X)}\cS_\alpha(X),\\
        \cR_\varepsilon(X) &= \cR_\alpha(X)-(1-\alpha)\varepsilon, \quad  \forall \; \alpha \in \cA_\varepsilon(X) \\
        \cD_\varepsilon(X) &= \cD_\alpha(X)- (1-\alpha)\varepsilon, \quad  \forall \; \alpha \in \cA_\varepsilon(X)\\
        \cV_\varepsilon(X) &= \bbe[|X|-\varepsilon]_+ +  \bbe [X],\\
        \cE_\varepsilon(X) &= \bbe[|X|-\varepsilon]_+  =\textrm{Vapnik error.}
    \end{aligned}
\end{equation*}
\end{Prop}
\begin{proof}
    See Appendix \ref{proof of prop 3.4}
\end{proof}
\begin{remark}[$\cR_\varepsilon$ and $\cD_\varepsilon$] The risk $\cR_\varepsilon(X)$ and deviation $\cD_\varepsilon(X)$ are single-valued functionals, therefore, given a random variable $X,$ functionals $\cR_\alpha(X) - (1-\alpha)\varepsilon$ and $\cD_\alpha(X) - (1-\alpha)\varepsilon$ have one value for any $\alpha \in \cA_\varepsilon(X).$
\end{remark}

\begin{remark}[Non-regularity of the Quantile Symmetric Average Quadrangle]\label{reg remark}
Note that in general, the Vapnik error $\cE_\varepsilon(X)  =  \bbe[|X|-\varepsilon]_+$ is not regular for each $\varepsilon \geq 0$, since it fails to satisfy the nonzeroness axiom (i.e., there exists $X \in \cL^2(\Omega), \ X \not\equiv  0$ such that $\cE_\varepsilon(X) = 0$). Indeed, $\cE_\varepsilon(X) = 0$ for all $X \in \cL^2(\Omega)$ such that $|X|\leq \varepsilon$ almost surely. 
On the other hand, CVaR Norm Quadrangle is regular and will play an important role in further analysis of SVR.
\end{remark}
\begin{remark}[Uniqueness of Statistic for Quantile Symmetric Average Quadrangle for Absolutely Continuous Random Variables] \label{singleton remark}For an absolutely continuous random variable $X$, the left and right quantiles coincide, i.e., $q_\alpha^+(X) = q_\alpha^-(X) = q_\alpha(X)$. Therefore, for $0 \leq \varepsilon < \frac{1}{2}(\ess \,X - \esi\, X)$, the set $\cA_\varepsilon(X) = \Lset \alpha \in [0,1) \Mset \varepsilon = \frac{1}{2}(q_{(1+\alpha)/2}(X) - q_{(1-\alpha)/2}(X))\Rset$ is a singleton, and consequently, the statistic of the Quantile Symmetric Average Quadrangle coincides with that of the CVaR Norm Quadrangle.
    
\end{remark}

\begin{comment}
\begin{remark}[Mixed Quantile Quadrangle, cf. \citep{Quadrangle}]
A statistic is not unique with respect to the choice of an error. Here is an example of a quadrangle  also having an average of the two symmetric quantiles as its statistic. Let $\alpha_1 = (1-\alpha)/2, \ \alpha_2 = 1-\alpha_1$
\begin{equation*}
    \begin{aligned}
        \cS(X) &= \frac{1}{2}\left(q_{\alpha_1}(X) + q_{\alpha_2}(X)\right),\\
        \cR(X) &= \frac{1}{2}\left(\hq_{\alpha_1}(X) + \hq_{\alpha_2}(X)\right),\\
        \cD(X) &= \frac{1}{2}\left(\hq_{\alpha_1}(X-\bbe X) + \hq_{\alpha_2}(X-\bbe X)\right),\\
        \cV(X) &= {\displaystyle \min_{B_1,B_2}}
           \Lset \frac{1}{2}\sum_{k=1}^2  \cV_\low{\alpha_k}(X-B_k)\Mset
            \frac{1}{2}\sum_{k=1}^2  B_k=0\Rset,\\
        \cE(X) &= {\displaystyle \min_{B_1,B_2}}
         \Lset \frac{1}{2}\sum_{k=1}^2 \cE_\low{\alpha_k}(X-B_k)\Mset
            \frac{1}{2}\sum_{k=1}^2 B_k=0\Rset,
    \end{aligned}
\end{equation*}
where  $\cV_\low{\alpha_k}(X)= \dfrac{1}{1-\alpha_k}EX_\plus,\\ \ \cE_\low{\alpha_k}(X)=
               E\left[\dfrac{\alpha_k}{1-\alpha_k}X_\plus +X_\minus\,\right].$
\end{remark}
\end{comment}

\section{SVR as a Generalized Regression}\label{sec: svr as generalized regression}

This section considers SVR as a regularized regression corresponding to the quadrangles defined in Propositions \ref{Prop Mafusalov}, \ref{Prop expect quadr}. Firstly, we establish an equivalence of two variants of SVR through the Dual CVaR Optimization Formula (Theorem \ref{dual squantile th}). Then we discuss the estimation properties of SVR. Further, we discuss SVR in the context of DRO and provide equivalent reformulations. Finally, we derive a dual formulation of the $\nu$-SVR and discuss its nonlinear extension using the well-known ``kernel trick''.

\subsection{Equivalence of $\texorpdfstring{\mathbold{\varepsilon}}{e}$-SVR and $\texorpdfstring{\mathbold{\nu}}{nu}$-SVR for Random Vectors}
Below we formulate the $\varepsilon$-SVR and $\nu$-SVR for stochastic vectors. 
We denote by $\mathbold{X} = (X_1, \ldots, X_n)^\top$ a vector of random variables (factors) $X_i \in \cL^2(\Omega), \ i =1,\ldots, n$ and by $Y  \in \cL^2(\Omega)$ a target random variable (regressant). This setting is more general, compared to the Section \ref{Sec SVR Formulations} because it does not assume that the random vectors have a fixed number of equally probable outcomes. Let us denote the linear regression residual by $$Z(\mathbold{w},b) = Y - (\mathbold{w}^\top\mathbold{X} +b), \quad (\mathbold{w},b) \in \bbr^{n+1}\,.$$
\begin{Def}[$\varepsilon$-SVR for random vectors]\label{stoch e-svr}
The $\varepsilon$-SVR (similar to \eqref{e-SVR with e-loss vector probabilistic})
 is stated as
\begin{equation}\label{e-SVR stochastic}
     \min_{\mathbold{w},b} \quad  \bbe\left[|Z(\mathbold{w},b)| - \varepsilon \right]_+ + \frac{\lambda}{2}\norm{\mathbold{w}}^2_2 \,.
\end{equation}
\end{Def}

\begin{Def}[$\nu$-SVR for random vectors]
\label{stoch nu-svr}
The $\nu$-SVR (similar to
 \eqref{nu-SVR with cvar norm vector probabilistic}) with $\nu = 1 - \alpha$ is stated as
\begin{equation}\label{nu-SVR stochastic}
     \min_{\mathbold{w},b} \quad \llangle Z(\mathbold{w},b) \rrangle_\alpha + \frac{\lambda}{2}\norm{\mathbold{w}}^2_2\,.
\end{equation}
\end{Def}

\vspace{10pt}
Given the equivalency of problems (\ref{e-SVR with e-loss vector probabilistic}) and (\ref{nu-SVR with cvar norm vector probabilistic}), the natural question is whether a similar statement is valid for (\ref{e-SVR stochastic}) and (\ref{nu-SVR stochastic}). The following Proposition \ref{dual svr connection} answers this question. 
\begin{Prop}[The $\varepsilon$-SVR \& $\nu$-SVR Equivalence]\label{dual svr connection}
Let $\lambda > 0$ be a regularization constant. Then 
\begin{itemize}
\item[\textit{(i)}]  
if  $(\mathbold{w}^*,b^*)$ is an optimal solution vector of (\ref{nu-SVR stochastic}) for some $\alpha \in [0,1)$, then 
 $(\mathbold{w}^*,b^*)$ is also an optimal solution vector of (\ref{e-SVR stochastic}) for each $\varepsilon \in q_\alpha(|Z(\mathbold{w}^*,b^*)|).$ 

\item[\textit{(ii)}] 
if $(\mathbold{w}^*,b^*)$ is an optimal solution vector of  (\ref{e-SVR stochastic}) for some
$\varepsilon \geq 0$, then 
 $(\mathbold{w}^*,b^*)$ is also an optimal solution vector of  (\ref{nu-SVR stochastic})
 for each 
 $$\alpha \in \big[\bbp(|Z(\mathbold{w}^*,b^*)|<\varepsilon),\bbp(|Z(\mathbold{w}^*,b^*)|\leq \varepsilon)\big) = \cI_\varepsilon. $$
    
\end{itemize}
\end{Prop}
\begin{proof}
    See Appendix \ref{proof of prop 4.2}
\end{proof}
\begin{remark}
Note that the regression residual $Z(\mathbold{w},b) = Y - (\mathbold{w}^\top\mathbold{X} + b)$ can have a more general form, i.e., $Z(\mathbold{w},b) = Y - f(\mathbold{w},\mathbold{X}),$ where $f \in \cF$ is a class of functions that is wider than the class of affine functions. In other words, Proposition \ref{dual svr connection} holds in a more general nonlinear setting (cf. Problem \ref{Generalized Regression Problem}). However, in the case of $\ell^2$ regularization, the class of affine functions is sufficient, since one may apply the ``kernel trick''.
\end{remark}

\subsection{The $\texorpdfstring{\mathbold{\varepsilon}}{e}$-SVR and $\texorpdfstring{\mathbold{\nu}}{nu}$-SVR as Deviation Minimization Problems}
This section studies $\varepsilon$-SVR and $\nu$-SVR as stochastic optimization problems through the concept of deviation measures. \cite{RiskTuning} proved a theorem, which relates the generalized regression problem with the minimization of deviation measures.

Further, we consider the formulation of  $\nu$-SVR and $\varepsilon$-SVR in the context of the generalized regression.

\begin{Prob}[Generalized Regression] \label{Generalized Regression Problem}
Given a random vector of factors $\mathbold{X} = (X_1, \ldots, X_n)^\top$ find a function from a given class  $f \in \mathcal{F}$, solving the following optimization problem for approximating a regressant $Y$
\begin{equation}\label{Generalized Regression}
   \min_{f\in \mathcal{F}} \quad \cE(Z_f), \ \text{where} Z_f = Y - f(\mathbold{X})\, .
\end{equation}
\end{Prob}
\begin{comment}
\begin{remark}[Parametric case and regularization]
Usually, a parametric class of functions $f(\mathbold{w},\mathbold{X})$ is specified, where $\mathbold{w} \in \mathcal{W}$ is a given set of parameters (e.g., $\bbr^d, \ d \in \bbn$). In this setting, it is common to include a regularization penalty into the objective of the generalized regression problem (\ref{Generalized Regression}), i.e.,
\begin{equation}\label{Regularized Generalized Regression}
 \min_{\mathbold{w} \in \mathcal{W}} \quad \cE(Z_\mathbold{w}) + \frac{\lambda}{2} \norm{\mathbold{w}}^2_2,
\end{equation}
where $Z_\mathbold{w} = Y - f_\mathbold{w}(\mathbold{X}), \ f\in \mathcal{F},  \ \lambda \in \bbr^+.$
\end{remark}
SVR  fits into this picture, since it is  a minimization of an  error measure with a regularization penalty, where $\cF$ is a class of affine functions, i.e., $f_{\mathbold{w},b}(\mathbold{X}) = \mathbold{w}^\top \mathbold{X} + b$ and $\mathcal{W} = \bbr^{n+1}$.
\end{comment}

Below is the theorem from \citep{Quadrangle} about solving the regression problem by minimizing a deviation.
\begin{Th}[Error Shaping Decomposition of Regression]\label{error shaping th}
Consider problem (\ref{Generalized Regression}) for $\mathbold{X}$ and $Y$ in the 
case of $\cE$ being a regular measure of error and $\cF$ being a class of 
functions $f:\bbr^n \to\bbr$ such that 
$$ 
     f\in \cF \implies f+C \in \cF \text{for all} C\in\bbr\,.
$$
Let deviation $\cD$ and statistic $\cS$ correspond to $\cE$.
Problem (\ref{Generalized Regression}) is equivalent to:
\begin{equation}\label{error decomp}
\begin{split}
     \min_{f\in\cF} &\quad \cD(Z_f)\\
     \textrm{s.t.} &\quad 
                   0\in \cS(Z_f)\,.
\end{split}
\end{equation} 
\end{Th}
\begin{Cor}[Linear Regression]
Consider the generalized regression problem (\ref{Generalized Regression}), where $\cF$ is a class of affine functions, i.e., $Z_f = Y - \mathbold{w}^\top\mathbold{X} - b.$ Denote by $\bar{Z}(\mathbold{w}) = Y - \mathbold{w}^\top\mathbold{X}$. Then $(\mathbold{w}^*, b^*)$ is an optimal solution vector of \eqref{Generalized Regression} if and only if 
$\mathbold{w}^*$ is an optimal solution to
\begin{equation}\label{error decomp lin}
\begin{aligned}
     \min_{\mathbold{w}} &\quad \D(\bar{Z}(\mathbold{w})) \, ,\\
\end{aligned}
\end{equation}
and $b^* \in \cS(\bar{Z}(\mathbold{w}^*))$.
\end{Cor}
Therefore, the Decomposition of Regression Theorem implies that the optimal intercept $b^*$ is a known function (statistic) of  $\bar{Z}(\mathbold{w}^*)$.

\begin{remark}[The choice of $\varepsilon$ in $\varepsilon$-SVR] Denote by $(\mathbold{w}^*_\varepsilon,b^*_\varepsilon)$ a solution vector of problem 
\eqref{e-SVR stochastic}.
Proposition \ref{dual svr connection} provides a constraint for parameter  $\varepsilon,$ i.e.,
\begin{equation}\label{selection of eps}
    0 \leq \varepsilon < \frac{1}{2} \, \big (\ess \, Z (\mathbold{w}^*_\varepsilon,b^*_\varepsilon)  - \esi\,Z(\mathbold{w}^*_\varepsilon,b^*_\varepsilon)\big )\,. 
\end{equation}
It can be proved that for sufficiently large $\varepsilon$ the vector $(\mathbold{w}^*_\varepsilon,b^*_\varepsilon) = \mathbf{0}$ is an optimal solution of the problem \eqref{e-SVR stochastic}, where $\mathbf{0}$ is the $(n+1)$-dimensional zero vector.
Therefore, $\varepsilon$ should satisfy the inequality  
\begin{equation}\label{selection of eps fin}
    0 \leq \varepsilon < \frac{1}{2} (\ess\, Y - \esi \,Y)\,.
\end{equation}
This bound for the parameter $\varepsilon$ was earlier recommended by
 \cite{NuSVR}.

%The largest $\varepsilon \geq 0$ is such that $(\mathbold{w}^*,b^*) = 0$\footnote{here 0 is the $(n+1)$-dimensional vector.} since the optimal objective value of \eqref{e-SVR stochastic} equals zero in this case\footnote{it is easy to verify using the necessary condition of extremum.}. Conversely, we note that 
%\begin{equation*}
%    0 \leq  \min_{\mathbold{w},b} \quad  \bbe\left[|Z(\mathbold{w},b)| - \varepsilon \right]_+ + \frac{\lambda}{2}\norm{\mathbold{w}}^2_2 \leq \bbe[|Y| -\varepsilon]_+
%\end{equation*}
%with $\bbe[|Y| -\varepsilon]_+ = 0$ iff $|Y| \leq \varepsilon$ a.s. or equivalently 
%\begin{align}
 % \varepsilon \geq \max\{Y,-Y\} \ a.s. &\iff \varepsilon \geq \ess(Y) \ \textrm{and} \ \varepsilon \geq - \esi(Y) \\
%  & \implies \varepsilon \geq 1/2(\ess(Y) - \esi(Y))\label{eps inequality}.
%\end{align} 
%Hence if $\bbe[|Y| -\varepsilon]_+ = 0$ then  $\varepsilon \geq 1/2(\ess(Y) - \esi(Y)),$ the optimal objective value of \eqref{nu-SVR stochastic} is zero and consequently $(\mathbold{w}^*, b^*) = 0.$
%Thus, 
%\begin{equation}\label{obj equals 0}
 % \min_{\mathbold{w},b} \quad  \bbe\left[|Z(\mathbold{w},b)| - \varepsilon \right]_+ + \frac{\lambda}{2}\norm{\mathbold{w}}^2_2 = 0 \iff   (\mathbold{w}^*, b^*) = 0.
%\end{equation}

%Therefore, given \eqref{obj equals 0} and \eqref{eps inequality}, \eqref{selection of eps} is replaced by
%\begin{equation}\label{selection of eps feasible}
 %   0 \leq \varepsilon < 1/2(\ess(Y) - \esi(Y)).
%\end{equation}
%Inequality \eqref{selection of eps feasible} is exactly the prescription recommended by \cite{NuSVR}.
\end{remark}

\begin{Cor}[Decomposition of SVR]\label{error shap coroll}
Let $Y - \mathbold{w}^\top\mathbold{X} =  \bar{Z}(\mathbold{w}) \in \cL^2(\Omega)$, $\lambda > 0$ be a regularization constant, and $(\cS_\varepsilon, \cD_\varepsilon)$ be the corresponding pairs of the statistic and deviation from CVaR Norm and Quantile Symmetric Average Quadrangles, respectively. Then
\begin{itemize}
\item [(\textit{i})] $\forall \, \alpha \in [0,1),$ $(\mathbold{w}^*,b^*)$ is an optimal solution vector of problem \eqref{nu-SVR stochastic} if and only if $\mathbold{w}^*$ is a solution to
    \begin{equation}\label{nu-SVR dev}
       \begin{aligned}
             \min_{\mathbold{w}} &\quad \cD_\alpha(\bar{Z}(\mathbold{w})) + \frac{\lambda}{2}\norm{\mathbold{w}}^2_2\,,\\
        \end{aligned}
    \end{equation}
   and $b^* \in \cS_\alpha(\bar{Z}(\mathbold{w}^*))$.
     \item [(\textit{ii})] 
     $\forall \, \varepsilon \in \big [0, \frac{1}{2} (\ess\, Y - \esi \,Y)\big),$ 
    % \alpha \in \cA_{\varepsilon}(\bar{Z}(\mathbold{w}^*)),$ 
     $(\mathbold{w}^*,b^*)$ is an optimal solution vector of problem \eqref{e-SVR stochastic} if and only if $\mathbold{w}^*$ is an optimal solution to
    \begin{equation}\label{e-SVR dev}
        \begin{aligned}
             \min_{\mathbold{w}} &\quad \cD_\varepsilon(\bar{Z}(\mathbold{w})) + \frac{\lambda}{2}\norm{\mathbold{w}}^2_2\,,\\
        \end{aligned}
    \end{equation}
    and  $b^* \in \cS_\varepsilon(\bar{Z}(\mathbold{w}^*))$.    
\end{itemize}
\end{Cor}
\begin{proof}
Items \textit{(i), (ii)} follow from the Theorem \ref{error shaping th}, since $\norm{\mathbold{w}}_2$ obviously does not depend on $b$.
\end{proof}

\subsection{Estimation}\label{sec: estimation}
This subsection discusses generalized regression in the risk quadrangle framework. Regression Problem \ref{Generalized Regression Problem} approximates a random variable $Y$ by a function $f(\mathbold{X}) \in \mathcal{F}$. By the regression being ``generalized'', we mean that the approximation error (residual) $Z_f=Y-f(\mathbold{X})$ is assessed by an error $\cE$. The function $f$ is called an \emph{estimator}, and the class of functions $\cF$ is called a \emph{class of estimators}. The function $\hat{f} \in \cF$ is called \emph{the best estimator} in the class $\cF$ w.r.t. error $\cE$ if it solves the problem (\ref{Generalized Regression}). The following Regression Theorem from \cite{Quadrangle} links the function $\hat{f}$ and a statistic $\cS$  corresponding to an error $\cE$.
\begin{Th}[Regression]\label{Regression Theorem}
Consider regression problem (\ref{Generalized Regression}) for a random vector $\mathbold{X}$ and $Y$ in the 
case of $\cE$ being a regular error and $\cF$ being a class of 
functions $f:\bbr^n \to\bbr$ such that 
$$ 
     f\in \cF \implies f+C \in \cF \text{for all} C\in\bbr\,.
$$
Let $\cD$ and $\cS$ correspond to $\cE$. Moreover let $\cE$ be of \textbf{expectation type} and let $\cF$ include a function $f$ 
satisfying 
\begin{equation}\label{regr}
  \eqalign{
   f(\mathbold{x})\in \cS\left(Y| \mathbold{x} \right) 
       \text{almost surely for $\mathbold{x} \in G$,} \cr
   \!\text{where} Y| \mathbold{x} = Y_\low{\mathbold{X} = \mathbold{x}}
       \text{(conditional distribution), } }  
\end{equation}
with $G$ being the support of the distribution in $\bbr^n$ induced by
$\mathbold{X}$.
\begin{comment}
\footnote{
    Almost surely, in (\ref{regr}), refers to this distribution.} 
\end{comment}

Then $f$ solves the regression problem and \textbf{estimates}
this conditional statistic
\begin{comment}
\footnote{
     It is assumed, for this part, that the distribution of 
     $Y(\mathbold{x})$ for $\mathbold{x}\in G$ belongs to 
     $\cL^2(\Omega)$, and the same then for the random variable 
     $Y(\mathbold{X})$ obtained from it.}
\end{comment}
in a sense that 
\begin{equation}
   f(\mathbold{X}) \in \cS\left(Y| \mathbold{X} \right) \text{almost surely.} 
\end{equation}
\end{Th}
Further, we refer to two classical examples illustrating the above theorem.
\begin{Ex}[Least Squares]
The Least Squares regression with the error, $\cE(X) = \bbe[X^2]$, is formulated as follows 
\begin{equation}
    \min_{f\in \cF} \quad \bbe[Z_f^2].
\end{equation}
By solving this problem, we obtain the best estimator $\hat{f}(\mathbold{X}) = \bbe\left[Y| \mathbold{X}\right] = \cS(Y|\mathbold{X})$  = conditional statistic,  which is a conditional mean, corresponding to the mean squared error $\cE$ in the Mean-based Quadrangle, cf. \citep{Quadrangle}.
\end{Ex}

\begin{Ex}[Quantile Regression]
\begin{equation}
    \min_{f\in \cF} \quad \bbe\left[\frac{\alpha}{1-\alpha}(Z_f)_+ + (Z_f)_-\right]. 
\end{equation}
The best estimator is $\hat{f}(\mathbold{X}) \in q_\alpha\left(Y| \mathbold{X}\right) = \cS(Y|\mathbold{X})$ = conditional statistic,  which is conditional quantile,  corresponding to the  $\cE$ in the Quantile Quadrangle, cf. \citep{ KoenkerBook, Quadrangle}.
\end{Ex}
Now, consider the Vapnik error, i.e., $\cE_\varepsilon(X) = \bbe[|X|-\varepsilon]_+, \ \varepsilon \geq 0$. According to Proposition \ref{Prop expect quadr}, the best estimator is
\begin{equation}\label{conditional statistic true}
\begin{aligned}
 \hat{f}(\mathbold{X}) \in \bigcup \limits_{\alpha \in \cA_\varepsilon(Z_{\hat{f}})} \Lset \dfrac{1}{2}\big(q_{(1-\alpha)/2}\big(Y| \mathbold{X}\big)+ q_{(1+\alpha)/2}\big(Y| \mathbold{X}\big)\big)\Rset\,.
\end{aligned}
\end{equation}
On the other hand, consider the CVaR norm, i.e., $\cE_\alpha(X) = \llangle X \rrangle_\alpha$. This error is not of expectation type. However, Proposition \ref{dual svr connection} implies (if conditions of regression theorem are satisfied) that for any $\alpha \in [0,1)$ we can pick an equivalent expectation type Vapnik error with 
$\varepsilon \in q_\alpha(|Z_{\hat{f}}|)$ such that \eqref{conditional statistic true} holds, where $\hat{f}$ is a solution to the generalized regression problem with CVaR norm error. Therefore, we choose to work with $\nu$-SVR. This choice is preferable for two reasons. First, the error measure, in this case, is a norm, simplifying the derivation of the dual formulation of $\nu$-SVR. Second, the choice of $\alpha$ is more intuitive and also plays an important role in $\nu$-SVR's interpretation as DRR.

Regarding the estimation properties of SVR,  adding the regularization penalty biases the estimator. In this case, the best estimator is a biased variant of \eqref{conditional statistic true}. However, when the regularization parameter tends to zero, as the sample size increases, we obtain an asymptotically unbiased estimator \eqref{conditional statistic true}, where $\cA_\varepsilon(Z_{\hat{f}})$ is possibly a singleton (cf. Remark \ref{singleton remark}). For instance, in the popular SVR solver LIBSVM, \citep{LIBSVM}, the regularization parameter $\lambda = \dfrac{1}{C l}$ and thus, $\lambda \to 0$ as $l \to \infty$. The following property provides a formal statement on this subject.

\begin{propt}[SVR Estimation]
Given the dataset \eqref{empirical data}, consider the following generalized regression problem
\begin{equation}\label{svr estimation eq}
    \min_{\mathbold{w},b} \quad \llangle Z(\mathbold{w},b) \rrangle_\alpha + \frac{\lambda}{2} \norm{\mathbold{w}}^2_2 \, ,
\end{equation}
where $\lambda \to 0$ as $l \to \infty$. Assume that the conditions of the Theorem \ref{Regression Theorem} are satisfied, $\cF$ is a class of affine functions, and $(\mathbold{w}^*,b^*)$ is a unique optimal solution to \eqref{svr estimation eq} when $\lambda = 0$.
Then the best estimator $\hat{f}_l(\mathbold{x}) = \mathbold{w}^{*\top}_l\mathbold{x} + b^*_l$ is an asymptotically unbiased estimator, i.e., 
\begin{equation}
    \lim_{l \to \infty} \bbe[(\mathbold{w}^*_l,b^*_l)] = (\mathbold{w}^*,b^*)
\end{equation}
and 
\begin{equation}
\hat{f}(\mathbold{x}) = \mathbold{w}^{*\top}\mathbold{x} + b^* \in   \dfrac{1}{2}\Big(q_{(1-\alpha^*)/2}\big(Y|\mathbold{X} = \mathbold{x}\big)+ q_{(1+\alpha^*)/2}\big(Y|\mathbold{X} = \mathbold{x}\big)\Big)\, ,
\end{equation}
 where $\alpha^* \in [0,1)$ is such that $$ \frac{1}{2}\big(q_{(1+\alpha^*)/2}(Z(\mathbold{w}^*,b^*)) - q_{(1-\alpha^*)/2}(Z(\mathbold{w}^*,b^*))\big) \ \  \bigcap \ \  q_\alpha(|Z(\mathbold{w}^*,b^*)|)  \, \neq  \, \varnothing. $$

\end{propt}
According to \cite{SmolaBook}, SVR can be considered as a generalized mean estimator. The estimation of the mean is accomplished by employing a cross-validation procedure for $\alpha$, using the MSE as a performance metric, cf. \cite{NewSVM}. This approach leverages the flexibility of averaging two symmetric quantiles, which proves to be a versatile statistic. For a wide range of distributions, particularly those that are symmetric or heavy-tailed, there exists an optimal value of $\alpha^*$ within the interval $[0,1]$ such that 
\begin{equation}\label{mean eqls avg}
    \dfrac{1}{2}\left(q_{(1-\alpha^*)/2}\left(Z_{\hat{f}_l}\right) + q_{(1+\alpha^*)/2}\left(Z_{\hat{f}_l}\right)\right) = \bbe[Z_{\hat{f}_l}] \, ,
\end{equation}
where $\hat{f}_l$ is an optimal estimator of problem \eqref{svr estimation eq}. The following Lemma \ref{lemma} states a more general result.

%However, it is important to note that while SVR can be viewed as a generalized mean estimator, there are cases where the relationship \eqref{mean eqls avg} may not hold. Further, we construct a counterexample demonstrating that SVR fails to accurately estimate the conditional mean of the estimated random value.

\begin{comment}
  \begin{Ex}[Counterexample for \eqref{mean eqls avg} 
 ??? Need to be corrected???]\label{counterexample}
Let $Z$ be an absolutely continuous random variable with $f_Z(z) = 4z^3$ being its density function on a $[0,1]$ interval. Simple calculations show that $\bbe[Z] = \frac{4}{5}$ and $q_\alpha(Z) =  \sqrt[4]{\alpha}, \ \alpha \in [0,1].$ The average of two symmetric quantiles $\frac{1}{2}\left(q_{(1-\alpha^*)/2}\left(Z\right) + q_{(1+\alpha^*)/2}\left(Z\right)\right) = \frac{1}{2}\left(\sqrt[4]{(1-\alpha^*)/2} + \sqrt[4]{(1+\alpha^*)/2}\right)$ is a monotonically decreasing function on $[0,1]$ with the maximum equal to $\frac{1}{\sqrt[4]{2}} < 0.9 = \bbe[Z]$ at $\alpha = 0.$ Therefore there is no $\alpha^* \in [0,1]$ such that \eqref{mean eqls avg} holds.
\end{Ex}  
\end{comment}

\begin{Lem}[Estimating any statistic with SVR] \label{lemma} 
Let $\cS: \cL^2(\Omega) \to \bbr$ (or in general a set-valued mapping  $\cS: \cL^2(\Omega) \rightrightarrows \bbr$) be an arbitrary statistic (e.g., mean, median, quantile, etc.) and let $f \in \cF$ be a class of affine functions satisfying the conditions of the Theorem \ref{Regression Theorem} with CVaR norm being the error function. Then for any $\alpha \in [0,1)$ 
 \begin{equation}\label{mean in avg}
    \cS(Y|\mathbold{X}) \ \  \bigcap  \ \ \dfrac{1}{2}\left(q_{(1-\alpha)/2}\left(Y| \mathbold{X}\right) + q_{(1+\alpha)/2}\left(Y| \mathbold{X}\right)\right)  \, \neq \,  \varnothing
 \end{equation} 
 if and only if 
 \begin{equation}\label{res in avg}
   0 \in \cS\left(Z_{\hat{f}}\right) , 
 \end{equation}
 where $\hat{f}$ is an optimal solution vector of \eqref{Generalized Regression Problem} with CVaR Norm error.
\end{Lem}
   \begin{proof}
       The proof is a direct implication of the \cite[Theorem 5.1]{rockafellar2015measures}.
   \end{proof} 
%We explain Lemma \ref{lemma} with the following example.
%\begin{Ex}[Estimating expectiles with SVR]
    
%\end{Ex}

\subsection{SVR as Distributionally Robust Regression}\label{sec:svr as drr}
The formulation of $\nu$-SVR as the regularized CVaR norm minimization problem \eqref{nu-SVR stochastic} admits an equivalent saddle point formulation leveraging the dual representation of CVaR, \citep{artzner1999coherent, Rockafellar2006}. Specifically, for any random variable $X \in \cL^2(\Omega, \cA, \bbp)$
\begin{equation}\label{dual cvar formulation}
 \hq_\alpha(X) = \max_{\bbq \in \cQ_\alpha} \quad \bbe_{\bbq}[X] \, ,   
\end{equation}
where $\cQ_\alpha = \left\{\bbq \ll  \bbp \mset 0 \leq \mathrm{d} \bbq/ \mathrm{d} \bbp \leq 1/(1-\alpha)     \right\},$ $\bbe_\bbq[\cdot]$ denotes the expectation w.r.t. a probability measure $\bbq,$ and $\mathrm{d} \bbq/ \mathrm{d} \bbp$ is a Radon--Nikodym derivative.

Consider the $\nu$-SVR regression problem \eqref{nu-SVR stochastic}. Denote by $\bbp_{(Y,\mathbold{X})}$ the joint distribution of the regressant $Y$ and factors $\mathbold{X}.$ Then given \eqref{dual cvar formulation}, problem \eqref{nu-SVR stochastic} can be equivalently rewritten as follows
\begin{equation}\label{nu svr dro stoch}
  \min_{\mathbold{w},b} \max_{\bbq_{(Y,\mathbold{X})} \in \cQ_\alpha}   \bbe_{\bbq_{(Y,\mathbold{X})}}[|Z(\mathbold{w},b)|] + \lambda \|\mathbold{w}\|_2^2 \, .
\end{equation}
Formulation \eqref{nu svr dro stoch} is a regularized distributionally robust optimization problem with the uncertainty (ambiguity) set $\cQ_\alpha,$ \citep{shapiro2017dro}. More concretely, \eqref{nu svr dro stoch} is a distributionally robust $\cL^1$-regression. In particular, for $\alpha = 0,$ \eqref{nu svr dro stoch} is precisely the regularized $\cL^1$-regression and when $\alpha \to 1,$ it is the regularized $\cL^\infty$-regression.

In practice, \eqref{nu svr dro stoch} is intractable since the joint distribution $\bbp_{(Y,\mathbold{X})}$ is frequently unknown, and instead the data \eqref{empirical data} is given. Therefore, one usually approximates the $\bbp_{(Y,\mathbold{X})}$ by empirical distribution (uniform distribution over the data) and solves
\begin{equation}\label{nu svr dro empirical}
   \min_{\mathbold{w},b} \max_{\mathbold{q} \in \cQ_\alpha} \quad  \sum_{i=1}^lq_i|z_i(\mathbold{w},b)| + \lambda \|\mathbold{w}\|_2^2 \, ,  
\end{equation}
where $z_i(\mathbold{w},b) = y_i - \mathbold{w}^\top\mathbold{x}_i -b, \ i = 1,\ldots,l$ and 
\begin{equation}\label{Q_alpha}
    \cQ_\alpha = \left \{\mathbold{q} \in \bbr^l \  \middle| \ \sum_{i=1}^l  q_i = 1, \ 0 \leq q_i \leq 1/l(1-\alpha)\right\}.
\end{equation}
Formulation \eqref{nu svr dro empirical} was considered in \cite{stable_regres}, where the set $\cQ_\alpha$ was replaced with
\begin{equation}\label{Q_k}
    \cQ_k = \left \{\mathbold{q} \in \bbr^l \  \middle| \ \sum_{i=1}^l  q_i = k, \ 0 \leq q_i \leq 1\right\}, \quad 0\leq k\leq l, \ k \in \mathbb{N}\;.
\end{equation}
It is obvious that \eqref{Q_alpha} and \eqref{Q_k} are equivalent for  $k =  \left \lfloor l(1-\alpha) \right \rfloor.$

Let $(\mathbold{w}^*, b^*, \mathbold{q}^*)$ be an optimal solution vector of \eqref{nu svr dro empirical}. Then the optimal $\mathbold{q}^* \in \cQ_\alpha$ in \eqref{nu svr dro empirical} is as follows (cf. \citep[pp. 524--525]{RoysetWets2021})
\begin{equation}\label{optimal q}
 q_i^*(\mathbold{w}^*,b^*) =   \left\{\begin{matrix}
 \frac{1}{l(1-\alpha)} &\text{if} |z_i(\mathbold{w}^*,b^*)| > q_\alpha(|\mathbold{z}(\mathbold{w}^*,b^*)|)\\
 \frac{P_q - \alpha}{m(1-\alpha)} &\text{if} |z_i(\mathbold{w}^*,b^*)| = q_\alpha(|\mathbold{z}(\mathbold{w}^*,b^*)|)\\
 0  &\text{ if}  |z_i(\mathbold{w}^*,b^*)| < q_\alpha(|\mathbold{z}(\mathbold{w}^*,b^*)|)
\end{matrix}\right.   
\end{equation}
with $P_q = \operatorname{prob}\{|\mathbold{z}(\mathbold{w}^*,b^*)| \leq q_\alpha(|\mathbold{z}(\mathbold{w}^*,b^*)|)  \}$ and $m$ is a number of observations $i$ such that $|z_i(\mathbold{w}^*,b^*)| = q_\alpha(|\mathbold{z}(\mathbold{w}^*,b^*)|).$ Evidently, \eqref{optimal q} is sparse.

Due to the inherent sparsity of $\mathbold{q}^*$, \cite{stable_regres} proposed formulation \eqref{nu svr dro empirical} as a superior optimization-based alternative to the traditional random train-test split for selecting training data. They conducted several case studies that numerically verify this approach yields better out-of-sample performance, as measured by MSE.

Given Lemma \ref{lemma}, this result is not surprising, as by appropriately adjusting $\alpha$, it is possible to approximate the conditional mean (the minimizer of MSE) with the conditional average of two symmetric quantiles, provided that such an $\alpha$ exists.

Moreover, assuming that \begin{equation}\label{true regression law} y_i = \mathbold{w}^{*\top}\mathbold{x}_i + b^* + \epsilon_i, \quad i = 1, 2, \ldots \end{equation} is the true regression model, where $\epsilon_i$ are i.i.d. random variables, it is possible to select an optimal $\alpha$ that minimizes the out-of-sample MSE, provided the distribution of $\epsilon_i$ is known and such $\alpha$ exists. Below, we provide a sequence of steps for the optimal $\alpha$ selection.
\begin{enumerate}
    \item Compute $\mu = \bbe[\epsilon_i];$
    \item Find $\alpha$ such that $ q_{(1+\alpha)/2}(\epsilon_i) + q_{(1-\alpha)/2}(\epsilon_i) = 2\mu$;
    \item Compute $x =  \Big(q_{(1+\alpha)/2}(\epsilon_i) - q_{(1-\alpha)/2}(\epsilon_i)\Big)/2;$
    \item Compute $\alpha^* = \operatorname{prob}\{|\epsilon_i|\leq x \};$
    \item Set $\nu = 1 -\alpha^*$ as a parameter of $\nu$-SVR or set $ \varepsilon = x$ as a parameter of $\varepsilon$-SVR.
    \item For $\epsilon_i$ with symmetric distribution, choose any $\alpha \in [0,1).$ 
\end{enumerate}
In general, the above prescription works in a nonlinear setting, i.e., when

$$y_i = f^*(\mathbold{x}_i) + \epsilon_i, \quad i = 1, 2, \ldots 
$$
To find the optimal $f^*$ in this case, one can employ the dual formulation of SVR with a ``kernel trick'' discussed in the subsequent section.

\subsection{Dual Formulation and Kernelization}\label{sec:dual formul}
This section provides a dual formulation of SVR and implements a ``kernel trick'', which generalizes SVR to a nonlinear case. 

In the following Proposition \ref{prop dual svr}, we consider a deterministic variant of the CVaR norm defined in \cite{bertsimas2011robust, CVaRNorm} and the equivalent primal formulation of $\nu$-SVR \eqref{nu-SVR with cvar norm vector probabilistic}, where $\lambda = 1/C, \ C>0$.

\begin{Prop}[Dual Formulation of SVR]\label{prop dual svr}
Let $\alpha = 1-\nu$ and
\begin{equation}\label{primal SVR}
\begin{aligned}
    \textbf{p}^\star := \min_{\mathbold{w},b,\mathbold{z}} \quad & C (1- \alpha) \llangle \mathbold{z} \rrangle^S_\alpha + \frac{1}{2}\norm{\mathbold{w}}^2_2\\
    \textrm{s.t.} \quad & \mathbold{z} =  \mathbold{y} - \hat{\mathbold{X}}\mathbold{w} - \textbf{1}_lb.
\end{aligned}
\end{equation}
be a primal SVR problem, where $\textbf{1}_l = (1,\ldots,1)^\top \in \bbr^l, \ \mathbold{y} = (y_1,\ldots, y_l)^\top,$ and $ \ \hat{\mathbold{X}} = (\mathbold{x}_1^\top,\ldots, \mathbold{x}_l^\top) \in \bbr^{l \times n}$. Then
\begin{equation}\label{dual SVR formulation l2}
\begin{aligned}
     \textbf{d}^\star := \max_{\mathbold{\mu}} \quad & \mathbold{\mu}^\top \mathbold{y} - \frac{1}{2} \mathbold{\mu}^\top \hat{\mathbold{X}} \hat{\mathbold{X}}^\top \mathbold{\mu}\\
    \text{s.t.} \quad & \norm{\mathbold{\mu}}_1 \leq C(1-\alpha), \ \norm{\mathbold{\mu}}_\infty \leq \frac{C}{l},\\
    & \mathbold{\mu}^\top \textbf{1}_l = 0.
\end{aligned}
\end{equation}
defines a dual SVR problem.
\end{Prop}
\begin{proof}
    See Appendix \ref{proof of prop 4.7}
\end{proof}

Introducing the kernel function $k: \bbr^n \times \bbr^n \to \bbr$ and noting that the objective function in \eqref{dual SVR formulation l2} depends on feature vectors only through their inner product, we define the kernel matrix 
\begin{equation*}
    \textbf{K} = k(\mathbold{x}_i,\mathbold{x}_j)_{i,j = 1}^l\,,
\end{equation*}
and substitute $\hat{\textbf{X}} \hat{\textbf{X}}^\top$ in \eqref{dual SVR formulation l2} with $\textbf{K}$, thus obtaining the nonlinear extension of SVR 
\begin{equation}\label{dual SVR formulation l2 kernel}
\begin{aligned}
     \max_{\mathbold{\mu}} \quad & \mathbold{\mu}^\top \mathbold{y} - \frac{1}{2} \mathbold{\mu}^\top \textbf{K} \mathbold{\mu}\\
    \text{s.t.} \quad & \norm{\mathbold{\mu}}_1 \leq C(1-\alpha), \ \norm{\mathbold{\mu}}_\infty \leq \frac{C}{l},\\
    & \mathbold{\mu}^\top \textbf{1}_l = 0.
\end{aligned}
\end{equation}
%Furthermore, noting that $\mathbold{\mu} = \mathbold{\mu}_+ - \mathbold{\mu}_-$ and $|\mathbold{\mu}| = \mathbold{\mu}_+ + \mathbold{\mu}_-$, where all operations are understood component-wise, \eqref{dual SVR formulation l2 kernel} can be written as follows
%\begin{equation}\label{dual SVR formulation l2 kernel equivalent}
%\begin{aligned}
%     \max_{\mathbold{\mu}_+, \mathbold{\mu}_-} \quad &  (\mathbold{\mu}_+ - \mathbold{\mu}_-)^\top \mathbold{y} - \frac{1}{2} ( \mathbold{\mu}_+ - \mathbold{\mu}_-)^\top \textbf{K} ( \mathbold{\mu}_+ - \mathbold{\mu}_-)\\
%    \text{s.t.} \quad & ( \mathbold{\mu}_+ + \mathbold{\mu}_-)^\top\textbf{1}_l \leq C(1-\alpha),\\
%    &(\mathbold{\mu}_+ - \mathbold{\mu}_-)^\top \textbf{1}_l = 0,\\
%    & 0\leq \mu^i_+,\mu^i_- \leq \frac{C}{l}, %\quad i = 1, \ldots, l.
%\end{aligned}
%\end{equation}
Let us compare \eqref{dual SVR formulation l2 kernel} with the  SVR dual formulation from \cite{NewSVM}
\begin{equation}\label{LIBSVM dual SVR}
\begin{aligned}
     \max_{\mathbold{\alpha}, \mathbold{\alpha^*}} \quad &  (\boldsymbol\alpha - \boldsymbol\alpha^*)^\top \mathbold{y} - \frac{1}{2} (\boldsymbol\alpha - \boldsymbol\alpha^*)^\top \textbf{K} (\boldsymbol\alpha - \boldsymbol\alpha^*)\\
    \text{s.t.} \quad & (\boldsymbol\alpha + \boldsymbol\alpha^*)^\top\textbf{1}_l \leq C(1-\alpha),\\
    &(\boldsymbol\alpha - \boldsymbol\alpha^*)^\top \textbf{1}_l = 0,\\
    & 0\leq \alpha_i, \alpha^*_i \leq \frac{C}{l}, \quad i = 1, \ldots, l.
\end{aligned}
\end{equation}

Notice that problems \eqref{dual SVR formulation l2 kernel} and \eqref{LIBSVM dual SVR} have equivalent primal formulations, \eqref{nu-SVR with cvar norm vector probabilistic} and \eqref{nu-SVR with e-loss vector probabilistic}, in the sense that from an optimal solution of one problem, an optimal solution for the other can be constructed. Therefore, dual problems are also equivalent. However, one may also notice that problem \eqref{LIBSVM dual SVR} has twice as many optimization variables as \eqref{dual SVR formulation l2 kernel}. Solvers such as Portfolio Safeguard\footnote{Download from \url{http://www.aorda.com/}} (PSG) that work directly with convex functions can benefit from problem statement \eqref{dual SVR formulation l2 kernel}. On the other hand, the popular SVR solver, LIBSVM, cf. \citep{LIBSVM}, works with \eqref{LIBSVM dual SVR}.

%Finally, the solution of the primal problem \eqref{primal SVR} can be expressed in terms of the solution of the dual problem, and the best estimator takes the following form
%\begin{equation}
%    \hat{f}_{\mathbold{w},b}(\mathbold{x}) = \sum_{i = 1}^{l}\mathbold{\mu}^\top k(\mathbold{x}_i,\mathbold{x}) + b,
%\end{equation}
%where the optimal $b$ can be obtained as follows
%\begin{equation}\label{optimal b}
%    b = y_i - k(\mathbold{x}_i,\mathbold{x}), \quad \forall \; i: |\mu^i|< \frac{C}{l}.
%\end{equation}

 %However, the most popular SVR solver such as LIBSVM, cf. \citep{LIBSVM} treats variables $\boldsymbol\alpha$ and $\boldsymbol\alpha^*$ separately, which leads to a two times larger optimization problem.
\section{Case Studies}\label{sec: case studies}
The following case study implements SVR for simulated data and numerically confirms the 

\begin{itemize}
    \item [\textit{(a)}]  equivalence between $\varepsilon$-SVR and $\nu$-SVR based on Proposition \ref{dual svr connection};
    \item [\textit{(b)}] error shaping decomposition of SVR based on Corollary \ref{error shap coroll}; 
    \item [\textit{(c)}] equivalence between the primal problem statement \eqref{primal SVR} and the dual problem statement \eqref{dual SVR formulation l2} of Proposition \ref{prop dual svr};
\end{itemize}
The case study results, data, and codes can be found at the following link\footnote{\url{http://uryasev.ams.stonybrook.edu/index.php/research/testproblems/advanced-statistics/support-vector-regression-risk-quadrangle-framework/}}. 

As a true law $f(x) = x, \ x \in [0,1]$ is chosen with $[0,1]$ interval being uniformly partitioned by points $x_i, \  i = 1, \ldots, l.$ Then depended variable is simulated as follows
\begin{equation*}
    y_i = x_i + \epsilon_i, \quad i = 1, \ldots, l\,,
\end{equation*}
where error terms $\epsilon_i \sim\textrm{Laplace}(0,1)$ are distributed  with density $$\rho(x;a,d) = \dfrac{1}{2d}\exp\left(-\dfrac{|x-a|}{2d}\right), \quad a = 0,  \ d = 1.$$

The PSG package is used to numerically implement SVR. Optimization problems in PSG are formulated with precoded analytical functions.

\subsection{Primal Problem Formulations.}
This section considers equivalent primal SVR problems. First, we minimized the error. Second, we minimized the corresponding deviation and obtained the same solution. 

\paragraph{Regularized Error Minimization.}
To numerically establish the equivalence between the $\nu$-SVR and $\varepsilon$-SVR, we fix $\alpha = 0.6, \ l = 1000, \ C = 1, \ \lambda = \dfrac{1}{2Cl},$ and solve the optimization problem \eqref{nu-SVR stochastic} with $\ell^2$ penalty. Then we set $\varepsilon = q_\alpha(|Z(\mathbold{w^*},b^*)|)$  by using the PSG function \texttt{var\_risk($\alpha,$ matrix)} (where \texttt{matrix} denotes the standard extended design matrix used to solve the regression problem) and solve \eqref{e-SVR stochastic}. Having the solution of \eqref{e-SVR stochastic}, we calculate the midpoint of the interval $\cI_{\varepsilon}$ from Proposition \ref{dual svr connection} with the PSG function \texttt{pr\_pen($\varepsilon,$ matrix)} and then set $\alpha_{new} = 1 -$\texttt{pr\_pen($\varepsilon,$ matrix)}. The equivalence follows from $\alpha \approx \alpha_{new}.$

\paragraph{Regularized Deviation Minimization.}
To numerically confirm the error shaping decomposition of SVR, we solve \eqref{nu-SVR dev} with the same parameters as for the error minimization. We first minimized the deviation and then calculated \texttt{var\_risk($\frac{1-\alpha}{2},$ matrix)} and \texttt{var\_risk($\frac{1+\alpha}{2},$ matrix)} separately in PSG. Finally, we set 
\begin{equation}\label{intercept}
\begin{aligned}
 b^* = \frac{1}{2}\big(\texttt{var\_risk($(1-\alpha)/2$, matrix)}
 + \texttt{var\_risk($(1+\alpha)/2$, matrix)}\big).
\end{aligned}
\end{equation}

\subsection{Dual Problem Formulations}
Consider the dual problem \eqref{dual SVR formulation l2}. To make this problem equivalent to the primal with $\lambda = \dfrac{1}{2Cl},$ we substitute $C$ with $Cl$ in \eqref{dual SVR formulation l2}. To speed up the calculations we replace the constraint $\norm{\mathbold{\mu}}_\infty \leq C$ with the so-called box constraint $-C \leq \mu^i \leq C, \ i = 1, \ldots, l.$ With these adjustments \eqref{dual SVR formulation l2} can be rewritten as 
\begin{equation}\label{dual SVR l2 numerical}
    \begin{aligned}
     \max_{\mathbold{\mu}} \quad & \mathbold{\mu}^\top \mathbold{y} - \frac{1}{2} \mathbold{\mu}^\top \hat{\textbf{X}} \hat{\textbf{X}}^\top \mathbold{\mu}\\
    \text{s.t.} \quad & \norm{\mathbold{\mu}}_1 \leq Cl(1-\alpha), \ \mathbold{\mu}^\top \textbf{1}_l = 0,\\
    & -C \leq \mu^i \leq C, \ i = 1, \ldots, l.
\end{aligned}
\end{equation}

To numerically confirm the equivalence between the primal and dual problem, we first solve the dual problem \eqref{dual SVR l2 numerical} in PSG and then set $\mathbold{w}^* = \hat{\mathbold{X}}^\top\mathbold{\mu}^*$ and calculate the intercept $b^*$ using \eqref{intercept}.
\subsection{Summary}
This section summarizes the results of all numerical experiments that have been conducted. 

\begin{table}[H]
    \centering
    \begin{tabular}{|l|c|c|c|c|c|}
    \hline
      Method & Uncertainty Measure  &  $b^*, \mathbold{w}^*$ & $\alpha$ & $\varepsilon$ & Solving Time (s)\\ \hline
      $\nu$-SVR (primal) & error  &  0.020089, 0.932221 & 0.6  & 0.914845 & 0.01\\ \hline
      $\varepsilon$-SVR (primal) & error & 0.020089, 0.932221 &0.600380 &0.914845  & 0.01\\ \hline
       $\nu$-SVR (primal)& deviation & 0.019983, 0.932221 &0.6 & 0.914739& 0.01\\ \hline
       $\nu$-SVR (dual)& error &0.019974, 0.932232& 0.6& 0.914740& 0.09\\ \hline

      \hline
    \end{tabular}
    \caption{Optimization outputs: SVR, primal and dual formulations.}
    \label{tab:al-eps}
\end{table}
Optimization outputs from Table \ref{tab:al-eps} numerically confirm the equivalence between SVR formulations.

\section{Conclusion}\label{sec:conclusion}
This paper formulated SVR in the RQ framework, establishing connections between this machine-learning tool and classical statistics, risk management, and DRO. A key contribution is the derivation of the quadrangle corresponding to $\varepsilon$-SVR (Proposition~\ref{Prop expect quadr}) revealing its risk, deviation, and statistic components.

We demonstrated that SVR is an asymptotically unbiased estimator of the average of two symmetric conditional quantiles. This result implies that by adjusting the parameter $\alpha = 1 - \nu$, the $\nu$-SVR can estimate various distributional statistics, including the mean, median, and expectiles. Moreover, the appropriate choice of performance metric during cross-validation, such as MSE for the mean or AMSE for expectiles -- allows for precise estimation of these statistics (Lemma \ref{lemma}) and understanding its limitations. %\ref{counterexample}.

Additionally, we have reformulated SVR as a deviation minimization problem within the RQ theory (Corollary \ref{error shap coroll}). This reformulation has practical implications for dimension reduction in linear regression, where the intercept can be analytically computed.

Another  result is the proof of equivalence between $\nu$-SVR and $\varepsilon$-SVR in a general stochastic setting (Proposition \ref{dual svr connection}). We provided analytical expressions for the parameters $\varepsilon$ and $\nu$ that establish this equivalence, unifying these two widely-used SVR formulations.

Furthermore, by applying duality theory of convex functionals within the RQ framework, we have reinterpreted $\nu$-SVR as a DRR problem, offering a novel perspective on SVR that has been underexplored in the literature.

Finally, we derived a new dual formulation of SVR that offers computational advantages by halving the number of variables compared to the standard dual formulation. This formulation is transparent, computationally efficient, and compatible with general-purpose optimization packages like CPLEX, Gurobi, CVX, and PSG, making it versatile for practical implementation.

 Theoretical results are validated with a case study.

\newpage
%\section{ }
\bibliography{references}
\bibliographystyle{plainnat}
  \renewcommand{\bibsection}{\subsubsection*{References}}

\appendix
\section{Appendices}
\subsection{Theoretical Background}\label{appendix1}

The fundamental risk quadrangle paradigm was developed in \cite{Quadrangle}. This framework established a connection between risk management, reliability, statistics, and stochastic optimization theories. In particular, the risk quadrangle theory provides a unified framework for generalized regression.
 \medskip
$$\eqalign{
   \hskip25pt \text{risk} \cR \,\longleftrightarrow \,\cD \text{deviation} &\cr
   \hskip55pt  \uparrow \hskip09pt \cS \hskip09pt \uparrow &\cr
   \hskip17pt \text{regret} \cV \,\longleftrightarrow \;\cE \text{error} &\cr
}$$
\smallskip

\centerline{ \bf Diagram~1:\quad The Fundamental Risk Quadrangle\hskip15pt }
\medskip
The risk quadrangle methodology united risk functions for a random value $X$ in groups (quadrangles) consisting of the following functions:
\begin{itemize}
    \item Risk $\cR(X)$, which provides a numerical surrogate for the overall hazard in $X$. 
    \item Deviation $\cD(X)$, which measures the ``nonconstancy'' in $X$ as its uncertainty. 
    \item Error $\cE(X)$, which measures the ``nonzeroness'' in $X$.  
    \item Regret $\cV(X)$, which measures the ``regret'' in facing the mix of outcomes of $X$. 
    \item Statistic $\cS(X)$ associated with $X$ through $\cE$ and $\cV$.
\end{itemize}
The following diagram contains general relationships among elements of the quadrangle:
\medskip
$$ \cD(X)= \min_C\!\lset\cE(X-C)\rset = \cR(X-\bbe X) $$

\vskip-12pt

$$ \cR(X)= \min_C\!\lset C+\cV(X-C)\rset = \bbe X+\cD(X) $$

\vskip-12pt

$$ \cS(X) =\argmin_C\!\lset\cE(X-C)\rset
          =\argmin_C\!\lset C+\cV(X-C)\rset $$

\vskip-12pt

$$ \cE(X) =\cV(X)-\bbe X,\qquad \cV(X)=\bbe X+\cE(X) $$

\centerline{ \bf Diagram~2:\quad The Relationship Formulae
\hskip15pt }
\medskip

\noindent
Here $\bbe X$ denotes the mathematical expectation of $X$, and the statistic, $\cS(X)$, can be a set if the minimum is achieved for multiple points.

The popular quantile quadrangle, cf. \cite{Quadrangle} is named after the quantile statistic. This quadrangle establishes relations between the CVaR optimization technique described in \cite{CVaR,CVaR2} and quantile regression, cf. \cite{KoenkerBassett}, \cite{KoenkerBook}. In particular, it was shown that CVaR minimization and quantile regression are similar procedures based on the quantile statistic in the regret and error representation of risk and deviation. 

\begin{Def}[Regular Regret Measure]\label{regular regrt measure}
A functional $\cV: \cL^2(\Omega) \to \bbr \cup \{+ \infty\} $ is called a \textit{regular measure of regret} if it satisfies the following axioms
\begin{itemize}
    \item[(V1)] \textbf{zero neutrality:} $\cV(0) = 0;$
     \item[(V2)] \textbf{convexity:} $\cV\left(\lambda X + (1-\lambda)Y\right) \leq \lambda \cV(X) + (1-\lambda)\cV(Y), \quad \forall \; X,Y$ and $\lambda \in [0,1]$;
    \item[(V3)] \textbf{closedness:} $\left\{ X \in \cL^2(\Omega)|\cV(X) \leq c\right\}$ is closed $\forall \; c < \infty$;
    \item[(V4)] \textbf{aversity:} $\cV(X) > \bbe X, \quad \forall \; X \neq const.$
\end{itemize}
\end{Def}

\begin{Def}[Regular Deviation Measure]\label{regular deviation measure}
A functional $\cD: \cL^2(\Omega) \to \bbr \cup \{+ \infty\} $ is called a \textit{regular measure of deviation} if it satisfies the following axioms
\begin{itemize}
    \item[(D1)] \textbf{constant triviality:} $\cD(C) = 0, \quad \forall \; C = const.;$
     \item[(D2)] \textbf{convexity:} $\cD\left(\lambda X + (1-\lambda)Y\right) \leq \lambda \cD(X) + (1-\lambda)\cD(Y), \quad \forall \; X,Y$ and $\lambda \in [0,1]$;
    \item[(D3)] \textbf{closedness:} $\left\{ X \in \cL^2(\Omega)|\cD(X) \leq c\right\}$ is closed $\forall \; c < \infty$;
    \item[(D4)] \textbf{nonzeroness:} $\cD(X) > 0, \quad \forall \; X \neq const.$
\end{itemize}
\end{Def}

\begin{Def}[Regular Risk Quadrangle] \label{risk quadrangle} A quartet $(\cR,\cD,\cV,\cE)$ of regular measures of risk, deviation, regret, and error is called a \emph{regular risk quadrangle} if it satisfies the relationship formulae in Diagram 2.
\end{Def}

\begin{Th}[Quadrangle Theorem] Let $X \in \cL^2$. Then \\
\vspace{-0.3cm}
\paritem{(a)} The relations $\cD(X) = \cR(X)-EX$ and $\cR(X)= EX+\cD(X)$
give a one-to-one correspondence between regular measures of risk $\cR$
and regular measures of deviation $\cD$.  In this correspondence, $\cR$ is
positively homogeneous if and only if $\cD$ is positively homogeneous.  On
the other hand, 
\begin{equation}\label{3.16}
  \text{$\cR$ is monotonic iff 
                     $\;\cD(X)\leq \sup X-EX\,$ for all  \, $X$.}
\end{equation}
  \paritem{(b)} The relations $\cE(X) = \cV(X)-EX$ and $\,\cV(X)= EX+\cE(X)$ 
give a one-to-one correspondence between regular measures of regret $\cV$
and regular measures of error $\cE$.  In this correspondence, $\cV$ is
positively homogeneous if and only if $\cE$ is positively homogeneous.  On
the other hand, 
\begin{equation}\label{3.17}
    \text{$\cV$ is monotonic if and only if 
                     $\;\cE(X)\leq |EX|\,$ for $X\leq 0$.}
\end{equation}  

 \paritem{(c)} For any regular measure of regret $\cV$, a regular
measure of risk $\cR$ is obtained by 
\begin{equation}
           \cR(X)=\min_C\!\Lset C+\cV(X-C) \Rset.  
\end{equation}

If $\cV$ is positively homogeneous, $\cR$ is positively homogeneous.
If $\cV$ is monotonic, $\cR$ is monotonic.
   \paritem{(d)} For any regular measure of error $\cE$, a regular
measure of deviation $\cD$ is obtained by 
  \begin{equation}
       \cD(X)=\min_C\!\Lset \cE(X-C) \Rset.  
  \end{equation}
If $\cE$ is positively homogeneous, $\cD$ is positively homogeneous.
If $\cE$ satisfies the condition in \eqref{3.17}, then $\cD$ satisfies the
condition in \eqref{3.16}.
   \paritem{(e)} In both (c) and (d), as long as the expression being
minimized is finite for some $C$, the set of $C$ values for which the
minimum is attained is a nonempty, closed, bounded interval.
     Typically this interval reduces to a single point.
Moreover, when $\cV$ and $\cE$ are paired as in (b), the interval comes out 
the same and gives the associated statistic: 

\begin{equation}
\begin{aligned}
 \argmin_C\!\lset C+\cV(X-C)\rset = \cS(X)=\argmin_C\!\lset\cE(X-C)\rset, 
 \;\text{with}\cS(X+C)=\cS(X)+C.
\end{aligned}
\end{equation}

\end{Th}

\subsection{Proof of Proposition \ref{Prop expect quadr}}\label{proof of prop 3.4}
\begin{proof}
Let us prove that the set 
\begin{equation*}
    \begin{aligned}
     \cA_\varepsilon(X) = \Lset \alpha \in [0,1)\Mset  &\frac{1}{2}(q^-_{(1+\alpha)/2}(X) - q^+_{(1-\alpha)/2}(X)) \leq \varepsilon\leq \frac{1}{2}(q^+_{(1+\alpha)/2}(X) - q^-_{(1-\alpha)/2}(X))\Rset
    \end{aligned}
\end{equation*}
is not empty for $0 \leq \varepsilon < (\ess(X)-\esi(X))/2.$ 

First, note that the maximum of the expression $\frac{1}{2}(q^+_{(1+\alpha)/2}(X) - q^+_{(1-\alpha)/2}(X))$ with respect to $\alpha$ occurs when $\alpha = 1$ and it is less or equal to $(\esi(X)-\ess(X))/2$ by Definition \ref{quantile}. The minimum with respect to $\alpha$ for both the left-hand and right-hand sides of the double inequality in $\cA_\varepsilon(X)$ occurs when $\alpha = 0$ and it equals zero. Thus, if $\varepsilon \notin [0, (\ess(X)-\esi(X))/2)$ the set $\cA_\varepsilon(X)$ is empty.

Second, since $q^+_{(1+\alpha)/2}(X) \geq q^-_{(1+\alpha)/2}(X)$ and $q^+_{(1-\alpha)/2}(X) \geq q^-_{(1-\alpha)/2}(X)$ for each $\alpha,$ the inequality 
$$q^-_{(1+\alpha)/2}(X) - q^+_{(1-\alpha)/2}(X)) \leq q^+_{(1+\alpha)/2}(X) - q^-_{(1-\alpha)/2}(X)
$$
always holds.

Therefore, for $0 \leq \varepsilon < (\ess(X)-\esi(X))/2$ there always exists $\alpha \in [0,1)$ such that the double inequality in $\cA_\varepsilon(X)$ holds, i.e, $\alpha \in \cA_\varepsilon(X).$ 

Now, let us prove the main statement of the proposition. Relying on the Theorem \ref{dual squantile th} consider the following equality
\begin{equation}
    \begin{aligned}
    \label{myprop dev}
        \min_{C}\lset  &\mathbb{E}[|X-C|-\varepsilon]_+ \rset= \min_{C} \max_{\alpha \in [0,1)}\left\{ \llangle X - C\rrangle_\alpha - (1-\alpha)\varepsilon \right\}.
\end{aligned}
\end{equation}
Then by Remark \ref{dual superquantile remark} and  Sion's minimax theorem, equality (\ref{myprop dev}) can be equivalently rewritten as follows (where the minimum is taken w.r.t. the extended real line, i.e., $C \in \bbr \cup \{+\infty, -\infty\}$)
\begin{equation}
    \begin{aligned}
        \min_{C}\lset  &\mathbb{E}[|X-C|-\varepsilon]_+ \rset = \max_{\alpha \in [0,1)}\min_{C} \left\{ \llangle X - C\rrangle_\alpha - (1-\alpha)\varepsilon \right\}.
\end{aligned}
\end{equation}
Furthermore, Proposition \ref{Prop Mafusalov} implies that for each $\alpha \in [0,1)$
\begin{equation}\label{myprop dev 1}
\begin{split}
   \min_{C}&\left\{ \llangle X - C \rrangle_\alpha - (1-\alpha)\varepsilon \right\}\\
   &= \frac{1}{2}\left((1+\alpha)\hq_{(1-\alpha)/2}(X) + (1-\alpha)\hq_{(1+\alpha)/2}(X)\right)
   - (1-\alpha)\varepsilon - \bbe X,  
\end{split}
\end{equation}
where optimal $C^* = \dfrac{1}{2}\left(q_{(1-\alpha)/2}(X) + q_{(1+\alpha)/2}(X)\right).$ By plugging (\ref{myprop dev 1}) in (\ref{myprop dev}), we get 
\begin{equation}\label{myprop dev 2}
\begin{split}
    & \min_{C}\lset  \mathbb{E}[|X-C|-\varepsilon]_+ \rset = \\
   &= \max_{\alpha \in [0,1)}\Lset   \frac{1}{2}\big((1+\alpha)\hq_{(1-\alpha)/2}(X)
   + (1-\alpha)\hq_{(1+\alpha)/2}(X)\big)- (1-\alpha)\varepsilon - \bbe X\Rset \\
   & = \max_{\alpha \in [0,1)}\Lset   \frac{1}{2}\big((1+\alpha)\hq_{(1-\alpha)/2}(X) + (1-\alpha)\hq_{(1+\alpha)/2}(X)\big) + \alpha \varepsilon\Rset - \varepsilon - \bbe X. \\
     \end{split}
\end{equation}
Denote 
\begin{equation}
\begin{aligned}
    \theta_\varepsilon(\alpha)& = \dfrac{1}{2}\big((1+\alpha)\hq_{(1-\alpha)/2}(X) + (1-\alpha)\hq_{(1+\alpha)/2}(X)\big) + \alpha \varepsilon.
    \end{aligned}
    \end{equation}
Then Remark \ref{dual superquantile remark} implies that $\theta_\varepsilon(\alpha)$ is a concave function
of $\alpha$. Thus $\alpha$ belongs to $\argmax\limits_\alpha \theta_\varepsilon(\alpha)$ if and only if 
\begin{equation}\label{extr cond}
    \dfrac{\partial^+ \theta_\varepsilon(\alpha)}{\partial \alpha} \leq 0 \leq \dfrac{\partial^-\theta_\varepsilon(\alpha)}{\partial \alpha}, \quad 0 \leq \varepsilon < (\ess(X)-\esi(X))/2.
\end{equation}
\cite{CVaR2} proved that 
\begin{equation}\label{cvar deriv rock}
    \frac{\partial^{\pm}}{\partial \alpha}\hq_\alpha(X) = \frac{1}{(1-\alpha)^2}\bbe[X-q_\alpha^{\pm}(X)]_+. 
\end{equation}
Hence \eqref{cvar deriv rock} implies 
\begin{equation}\label{cvar derivative}
    \frac{\partial^{\pm}}{\partial \alpha}((1+\alpha)\hq_{(1-\alpha)/2}(X)) =\hq_{(1-\alpha)/2}(X) - \frac{2}{1+\alpha}\bbe[X-q_{(1-\alpha)/2}^{\mp}(X)]_+.
\end{equation}
Theorem \ref{squantile opt th} implies
\begin{equation*}
    \hq_{(1-\alpha)/2}(X) = C + \frac{2}{1+\alpha}\bbe[X-C]_+,
\end{equation*}
for any $C \in q_{(1-\alpha)/2}(X) = [q^-_{(1-\alpha)/2}(X), q^+_{(1-\alpha)/2}(X)].$ Thus \eqref{cvar derivative} can be written as follows
\begin{equation*}
\begin{split}
    \frac{\partial^{\pm}}{\partial \alpha}((1+\alpha)\hq_{(1-\alpha)/2}(X)) &=\hq_{(1-\alpha)/2}(X) - \frac{2}{1+\alpha}\bbe[X-q_{(1-\alpha)/2}^{\mp}(X)]_+\\
    &= q^{\pm}_{(1-\alpha)/2}(X) + \frac{2}{1+\alpha}\bbe[X-q_{(1-\alpha)/2}^{\mp}(X)]_+ - \frac{2}{1+\alpha}\bbe[X-q_{(1-\alpha)/2}^{\mp}(X)]_+\\
    &= q^{\mp}_{(1-\alpha)/2}(X). 
\end{split}   
\end{equation*}
Similarly, 
\begin{equation*}
    \frac{\partial^{\pm}}{\partial \alpha}((1-\alpha)\hq_{(1+\alpha)/2}(X)) = -\hq_{(1+\alpha)/2}(X) + \frac{2}{1-\alpha}\bbe[X-q_{(1+\alpha)/2}^{\pm}(X)]_+= -q_{(1+\alpha)/2}^{\pm}(X).
\end{equation*}
Therefore, \eqref{extr cond} is equivalent to 
\begin{equation*}
\begin{aligned}
 &\frac{1}{2}(q^-_{(1+\alpha)/2}(X) - q^+_{(1-\alpha)/2}(X)) \leq \varepsilon
 \leq \frac{1}{2}(q^+_{(1+\alpha)/2}(X) - q^-_{(1-\alpha)/2}(X)).
\end{aligned}
\end{equation*}
which can be rewritten as
\begin{equation}
\label{myprop extrema condition}
 \varepsilon \in \frac{1}{2}\bigl (q_{(1+\alpha)/2}(X) - q_{(1-\alpha)/2}(X) \bigr )\;.
\end{equation}
For $\quad 0 \leq \varepsilon < (\ess(X)-\esi(X))/2\;\;$ let
\begin{equation*}
    \begin{aligned}
     \cA_\varepsilon(X) = \Lset \alpha \in [0,1)\Mset  & \varepsilon \in \frac{1}{2}\bigl (q_{(1+\alpha)/2}(X) - q_{(1-\alpha)/2}(X) \bigr )\Rset
    \end{aligned}
\end{equation*}
be a set of points satisfying (\ref{myprop extrema condition}). Then (\ref{myprop dev 2}) implies  
\begin{equation}\label{my prop dev fin}
    \begin{split}
    \cD_\varepsilon(X) &=  \min_{C}\lset \cE_\varepsilon(X-C)\rset\\
    &= \min_{C}\lset  \mathbb{E}[|X-C|-\varepsilon]_+ \rset \\
   &=    \frac{1}{2}\left((1+\alpha)\hq_{(1-\alpha)/2}(X) + (1-\alpha)\hq_{(1+\alpha)/2}(X)\right)-  (1-\alpha)\varepsilon -\bbe X, \quad \alpha \in \cA_\varepsilon(X),  
   \end{split}
\end{equation}
where $$\cS_\varepsilon(X) = \bigcup\limits_{\alpha \in \cA_\varepsilon}\Lset\dfrac{1}{2}\left(q_{(1-\alpha)/2}(X) + q_{(1+\alpha)/2}(X)\right)\Rset$$
is the minimizer for (\ref{my prop dev fin}).

Finally, $\cR(X) = \cD(X) + \bbe X$ and   $\cV(X) = \cE(X) + \bbe X$ imply a complete quadrangle quartet.
\end{proof}
\subsection{Proof of Proposition \ref{dual svr connection}}\label{proof of prop 4.2}
\begin{proof}

Theorem \ref{dual squantile th} implies that
\begin{equation}\label{dual SVR equation}
\min_{\mathbold{w}, b} \bbe\left[|Z(\mathbold{w},b)| - \varepsilon \right]_+ + \dfrac{\lambda}{2}\norm{\mathbold{w}}^2_2= \min_{\mathbold{w}, b}\max_{\alpha \in [0,1)} \Lset\llangle Z(\mathbold{w},b) \rrangle_\alpha - (1-\alpha)\varepsilon + \dfrac{\lambda}{2}\norm{\mathbold{w}}^2_2 \Rset.
\end{equation}
The left-hand side of \eqref{dual SVR equation} is a convex optimization problem. The existence of the optimal solution is guaranteed by the convexity, lower semi-continuity, and coercivity (i.e., for a fixed $\varepsilon, \ \bbe\left[|Z(\mathbold{w},b)| - \varepsilon \right]_+ + \dfrac{\lambda}{2}\norm{\mathbold{w}}^2_2 \to \infty$ as $\|\mathbold{w}\|_2 + b^2 \to \infty$) of the objective. The right-hand side of \eqref{dual SVR equation} is a minimax (convex-concave) optimization problem. Since the existence of an optimal solution holds for the left-hand side of \eqref{dual SVR equation} then the function $\llangle Z(\mathbold{w},b) \rrangle_\alpha - (1-\alpha)\varepsilon + \dfrac{\lambda}{2}\norm{\mathbold{w}}^2_2$  possesses a saddle point $(\alpha^*, (\mathbold{w}^*, b^*))$ on $[0,1) \times \bbr^{n+1}.$

Let $(\mathbold{w}^*,b^*)$ be an optimal solution to \eqref{e-SVR stochastic} for some $\varepsilon \geq 0$. Then \eqref{dual SVR equation} implies that 
$$\begin{aligned}
 &\bbe\left[|Z(\mathbold{w}^*,b^*)| - \varepsilon \right]_+ + \dfrac{\lambda}{2}\norm{\mathbold{w}^*}^2=\llangle Z(\mathbold{w}^*,b^*) \rrangle_{\alpha^*} - (1-\alpha^*)\varepsilon + \dfrac{\lambda}{2}\norm{\mathbold{w}^*}^2,
\end{aligned}
$$
where $(\alpha^*, (\mathbold{w}^*, b^*))$ is a saddle point for the right-hand side of \eqref{dual SVR equation}. Therefore, by the definition of $(\alpha^*, (\mathbold{w}^*, b^*))$
$$(\mathbold{w}^*, b^*) \in \argmin_{\mathbold{w},b} \llangle Z(\mathbold{w},b) \rrangle_{\alpha^*} + \dfrac{\lambda}{2}\norm{\mathbold{w}}^2,
$$
where $\alpha^* \in \left[\bbp(|Z(\mathbold{w}^*,b^*)|<\varepsilon),\bbp(|Z(\mathbold{w}^*,b^*)|\leq \varepsilon)\right)$ by Theorem \ref{dual squantile th}, which completes the proof of \textit{(ii)}.

Now, fix $\alpha^* \in [0,1)$ and consider the optimization problem (\ref{nu-SVR stochastic}). Let $(\mathbold{w}^*,b^*)$ be an optimal solution of (\ref{nu-SVR stochastic}).  Then Theorem \ref{dual squantile th} implies that for each $\varepsilon \in q_{\alpha^*}(|Z(\mathbold{w}^*,b^*)|)$ 
$$\begin{aligned}
 &\bbe\left[|Z(\mathbold{w}^*,b^*)| - \varepsilon \right]_+ + \dfrac{\lambda}{2}\norm{\mathbold{w}^*}^2=\llangle Z(\mathbold{w}^*,b^*) \rrangle_{\alpha^*} - (1-\alpha^*)\varepsilon + \dfrac{\lambda}{2}\norm{\mathbold{w}^*}^2,
\end{aligned}
$$
where $(\alpha^*, (\mathbold{w}^*, b^*))$ is a saddle point for the right-hand side of \eqref{dual SVR equation}. Therefore, 
$$(\mathbold{w}^*, b^*) \in \argmin_{\mathbold{w},b} \bbe\left[|Z(\mathbold{w},b)| - \varepsilon \right]_+ + \dfrac{\lambda}{2}\norm{\mathbold{w}}^2
$$
for all $\varepsilon \in q_{\alpha^*}(|Z(\mathbold{w}^*,b^*)|),$ which completes the proof of \textit{(i)}. 
\end{proof}

\subsection{Proof of Proposition \ref{prop dual svr}}\label{proof of prop 4.7}
Before directly going to the derivation of the dual formulation, let us introduce a couple of definitions and notations.

Let $\mathbb{X}$ be a normed space over $\bbr$ with norm $\norm{\cdot}$ (i.e., $\norm{x} \in \bbr$ for $x \in \mathbb{X}$). Then, the dual space denoted by $\mathbb{X}^*$ is defined as the set of all continuous linear functionals from $\mathbb{X}$ into $\bbr$.
\begin{Def}[Dual norm]
 For $f \in \mathbb{X}^*$, the \emph{dual norm}, denoted by $\norm{\cdot}_*$ of $f$ is defined by 
\begin{equation}
\begin{aligned}
 \norm{f}_* &= \sup \lset |f(x)|: x \in \mathbb{X}, \ \norm{x} \leq 1\rset\\
    &= \sup \left \{ \frac{|f(x)|}{\norm{x}}: x \in \mathbb{X},  \ x \neq 0  \right \}.
\end{aligned}
\end{equation}
\end{Def}
\begin{Def}[Conjugate function, \cite{Boyd}]
Let $f:\mathbb{X} \to \bbr \cup \{\infty\}$. Then a function on $\mathbb{X}^*$, defined by the following equality
\begin{equation}
    f^*(y) = \sup_{x \in \operatorname{dom} f} \{ \langle y , x\rangle - f(x) \},
\end{equation}
is called a \emph{conjugate} of $f$ or the \emph{Legendre--Young--Fenchel transform}.
\end{Def}
\begin{proof}
Define the Lagrangian
\begin{equation*}
    \begin{aligned}
   \cL(\mathbold{w},b,\mathbold{z},\mathbold{\mu})  & = C (1- \alpha) \llangle \mathbold{z} \rrangle^S_\alpha+ \frac{1}{2}\norm{\mathbold{w}}^2_2 + \mathbold{\mu}^\top( \mathbold{y} - \hat{\textbf{X}}\mathbold{w} - \textbf{1}_lb - \mathbold{z})\\
   & = - \left(\mathbold{\mu}^\top\mathbold{z} -  C (1- \alpha) \llangle \mathbold{z} \rrangle^S_\alpha\right)- \left((\hat{\textbf{X}}^\top\mathbold{\mu})^\top\mathbold{w} - \frac{1}{2}\norm{\mathbold{w}}^2_2\right)+ \mathbold{\mu}^\top\mathbold{y} - \mathbold{\mu}^\top \textbf{1}_lb,
\end{aligned}
\end{equation*}
where $\mathbold{\mu} = (\mu^1, \ldots, \mu^l)^\top$ is a Lagrange multiplier vector. 
Then 
\begin{equation*}
\begin{aligned}
       \min_{\mathbold{w},b,\mathbold{z}} \quad& \cL(\mathbold{w},b,\mathbold{z}, \mathbold{\mu})\\
       &= \min_{b} \quad  - C(1-\alpha) \left(\llangle \mathbold{\mu}/C(1-\alpha) \rrangle^{S}_\alpha \right)^*- \left (\frac{1}{2} \norm{\hat{\textbf{X}}^\top\mathbold{\mu}}^{2}_2\right)^* + \mathbold{\mu}^\top\mathbold{y}- \mathbold{\mu}^\top \textbf{1}_lb,
\end{aligned}
\end{equation*}
 which leads to the implicit constraint $\mathbold{\mu}^\top \textbf{1}_l = 0$. Note that in general, cf. \citep{Boyd}, for $\mathbold{x} \in \bbr^n$
\begin{equation*}
    \norm{\mathbold{x}}^* = \begin{cases}
  0 \quad & \norm{\mathbold{x}}_* \leq 1\\
 \infty \quad & \text{otherwise,}
\end{cases}
\end{equation*}
and
\begin{equation*}
    \left(\frac{1}{2}\norm{\mathbold{x}}^2\right)^* = \frac{1}{2}\norm{\mathbold{x}}^2_*.
\end{equation*}
Hence
\begin{equation*}
\begin{aligned}
 - C(1-\alpha)& \left(\llangle \mathbold{\mu}/C(1-\alpha) \rrangle^{S}_\alpha \right)^*=  
  \begin{cases}
  0 \quad & \llangle \mathbold{\mu} \rrangle^S_{\alpha *} \leq C(1-\alpha)\\
 -\infty \quad & \text{otherwise,}
\end{cases}
\end{aligned}
\end{equation*}
and 
\begin{equation*}
    \left (\frac{1}{2} \norm{\hat{\textbf{X}}^\top\mathbold{\mu}}^{2}_2\right)^* = \frac{1}{2}\norm{\hat{\textbf{X}}^\top\mathbold{\mu}}^{2}_{2} =\mathbold{\mu}^\top\hat{\textbf{X}}\hat{\textbf{X}}^\top\mathbold{\mu} .
\end{equation*}
Finally, noting that (cf. \citep{bertsimas2011robust, CVaRNorm2}), 
\begin{equation*}
    \llangle \mathbold{\mu} \rrangle^S_{\alpha *} = \max \lset \norm{\mathbold{\mu}}_1, l(1-\alpha)\norm{\mathbold{\mu}}_\infty\rset
\end{equation*}
and proceeding with the maximization of the Lagrangian w.r.t. dual variables $\mathbold{\mu}$ completes the proof.
\end{proof}
\end{document}